\documentclass[11pt]{article}

\usepackage[margin=1in]{geometry}

\usepackage{lipsum}

\usepackage[utf8]{inputenc} 
\usepackage[T1]{fontenc}    
\usepackage{tgtermes}

\usepackage{url}            
\usepackage{booktabs}       
\usepackage{amsfonts}       
\usepackage{nicefrac}       
\usepackage{microtype}      

\usepackage[english]{babel}

\usepackage{amsmath}
\usepackage{enumitem}
\usepackage{amssymb}
\usepackage{graphicx}
\usepackage{bm}
\usepackage{theorem}
\usepackage[colorlinks=true, allcolors=blue]{hyperref}

\usepackage{mathtools}
\newenvironment{proof}{\paragraph{\it Proof.}}{\hfill$\square$}

\usepackage[title]{appendix}
\usepackage{comment}
\usepackage{amsfonts}
\usepackage{dsfont}
\usepackage{hyperref}
\usepackage{lipsum}
\usepackage{dsfont,subcaption,float}
\usepackage{algorithm}
\usepackage{algorithmic}

\usepackage{thm-restate}
\usepackage[size=small,labelfont=bf]{caption}

\usepackage{xcolor}
\newcommand{\algname}{\texttt{EM$^2$}\ }
\newcommand{\yon}[1]{\textcolor{magenta}{}}
\newcommand{\yonDel}[1]{\textcolor{blue}{}}

\theoremstyle{plain}
\newtheorem{theorem}{Theorem}[section]

\newtheorem{lemma}[theorem]{Lemma}
\newtheorem{corollary}[theorem]{Corollary}
\theoremstyle{definition}
\newtheorem{definition}[theorem]{Definition}
\newtheorem{assumption}[theorem]{Assumption}
\theoremstyle{remark}

\newtheorem{remark}[theorem]{Remark}

\newcommand{\poly}{\mathrm{poly}}

\newcommand{\KL}{\textrm{KL}} 

\newcommand{\mS}{\mathcal{S}}	
\newcommand{\mA}{\mathcal{A}}	
\newcommand{\mT}{\mathcal{T}}	
\newcommand{\mY}{\mathcal{Y}}	
\newcommand{\mP}{\mathcal{P}}	
\newcommand{\mM}{\mathcal{M}}	

\newcommand{\mX}{\mathcal{X}}	
\newcommand{\mI}{\mathcal{I}}	
\newcommand{\mZ}{\mathcal{Z}}	
\newcommand{\indic}[1]{\mathds{1}\left\{ #1 \right\}} 

\newcommand{\PP}{\mathds{P}}    

\newcommand{\Eps}{\mathcal{E}}

\newcommand{\Exs}{\mathbb{E}}

\newcommand{\tssum}{\textstyle \sum}
\newcommand{\moment}{\mathbf{M}}

\allowdisplaybreaks

\newif\ifdraft
\newcommand{\blue}[1]{\textcolor{blue}{#1}}
\newcommand{\red}[1]{\textcolor{red}{#1}}

\newcommand{\jycomment}[1]{\ifdraft {\bf{{\red{{Jeongyeol --- #1}}}}}\else\fi}

\newcommand{\YE}[1]{\ifdraft {\bf{{\blue{{YE --- #1}}}}}\else\fi}

\title{\bf{\LARGE{Reward-Mixing MDPs with a Few Latent Contexts are Learnable}}}

\usepackage{authblk}
\author[1]{Jeongyeol Kwon \thanks{Most work is done while the author is at The University of Texas at Austin.}}
\author[2]{Yonathan Efroni} 
\author[3]{Constantine Caramanis}
\author[4]{Shie Mannor}
\affil[1]{Wisconsin Institute for Discovery, UW-Madison}
\affil[2]{Meta, New York}
\affil[3]{Department of Electrical and Computer Engineering, University of Texas at Austin}
\affil[4]{Department of Electrical Engineering, Technion / NVIDIA}

\begin{document}
\maketitle


\begin{abstract}
    We consider episodic reinforcement learning in reward-mixing Markov decision processes (RMMDPs): at the beginning of every episode nature randomly picks a latent reward model among $M$ candidates and an agent interacts with the MDP throughout the episode for $H$ time steps. Our goal is to learn a near-optimal policy that nearly maximizes the $H$ time-step cumulative rewards in such a model. Previous work \cite{kwon2021reinforcement} established an upper bound for RMMDPs for $M=2$. In this work, we resolve several open questions remained for the RMMDP model. For an arbitrary $M\ge2$, we provide a sample-efficient algorithm--\algname--that outputs an $\epsilon$-optimal policy using $\tilde{O} \left(\epsilon^{-2} \cdot S^d A^d \cdot \poly(H, Z)^d \right)$ episodes, where $S, A$ are the number of states and actions respectively, $H$ is the time-horizon, $Z$ is the support size of reward distributions and $d=\min(2M-1,H)$. Our technique is a higher-order extension of the method-of-moments based approach proposed in \cite{kwon2021reinforcement}, nevertheless, the design and analysis of the \algname algorithm  requires several new ideas beyond existing techniques. We also provide a lower bound of $(SA)^{\Omega(\sqrt{M})} / \epsilon^{2}$ for a general instance of RMMDP, supporting that super-polynomial sample complexity in $M$ is necessary. 
\end{abstract}

\section{Introduction}
Reinforcement learning (RL) in partially observable systems is a challenging problem. While partially observable Markov decision process (POMDP) is a versatile framework, POMDPs are generally hard to learn, primarily because the optimal policy depends on the entire history of the process \cite{smallwood1973optimal, krishnamurthy2016pac}. Due to its fundamental hardness, it is important to consider sub-classes of POMDPs that allow tractable solutions for a variety of applications. We are interested in a special and prevalent sub-class of POMDPs where the latent (unobservable) parts of the system remain static in each episode.

Specifically, we consider the framework of Latent MDPs (LMDPs), which has been studied in a few several works ({\it e.g.,} \cite{chades2012momdps, brunskill2013sample, hallak2015contextual, steimle2018multi, kwon2021rl}). In LMDPs, one MDP is randomly chosen from $M$ possible candidate models at the beginning of every episode, and an agent interacts with the chosen MDP for $H$ time steps of an episode. However, the identity of the chosen MDP is unknown to the agent, which we call the {\it latent contexts}. To learn near-optimal policies with latent contexts, existing POMDP solutions would require strong assumptions on reachability of the system ({\it e.g.,} \cite{azizzadenesheli2016reinforcement, guo2016pac}) or certain separability assumptions ({\it e.g.,} see conditions proposed in \cite{liu2022partially, golowich2022learning}). The work in \cite{kwon2021rl} does not give a satisfactory solution either since it requires a similar assumption of strong separability between latent contexts. However, these assumptions do not necessarily align with the applications ({\it e.g.,} dynamic web application \cite{hallak2015contextual}, medical treatment \cite{steimle2018multi}, transfer learning \cite{brunskill2013sample}) that we want to tackle with the proposed framework.

To take a step forward, the work in \cite{kwon2021reinforcement} develops a sample-efficient algorithm in the special case of two reward-mixing MDPs (RMMDPs): if the state transition models are shared across different MDPs, and further if $M=2$ with uniform priors, then there exists an efficient algorithm for learning a near-optimal policy {\it without any} further assumptions on system dynamics, {\it i.e.,} without reachability and separability. However, their results only apply to the case of $M=2$ and only to uniform priors over contexts \jycomment{with Bernoulli reward structures}.

\subsection{Our Contributions}
In this work we resolve several open questions for learning near-optimal policies in RMMDPs with $M\ge2$. We summarize our main results as follows:
\begin{enumerate}
    \item For general instances of RMMDPs with $M > 2$, without any further assumptions, we show that an $\epsilon$-optimal policy can be learned after exploring  $\tilde{O}(\poly(M,H) \cdot SA)^{\min(2M-1,H)} / \epsilon^2$ episodes (see Theorem \ref{theorem:sample_complexity_upper_bound} for the exact upper bound). 
    \item As a by-product of our analysis, we give a strictly improved result for learning RMMDPs with~$M=2$.  Prior work~\cite{kwon2021reinforcement} designed an algorithm with $\poly(H) \cdot \tilde{O}(S^2A^2/\epsilon^4)$ sample complexity, whereas we design an improved algorithm with  $\poly(H) \cdot \tilde{O}(S^2A^2/\epsilon^2)$ sample complexity.
    \item When all reward probabilities are strict integrals of the base probability, we show that the exponent of $S$ and $A$ can be significantly improved from $O(M)$ to $O(\log M)$. Examples of such cases include the case when all rewards are deterministic conditioned on latent contexts. 
    \item For general instances of RMMDPs, we show the lower bound of $(SA)^{\Omega(\sqrt{M})} / \epsilon^2$, which in part justifies that a super-polynomial number of samples in $M$ are necessary. 
\end{enumerate}
We note that our dependence on $\epsilon$ is tight. This is in contrast to \cite{kwon2021reinforcement} where the sample upper bound for the $M=2$ case scales with $\epsilon^{-4}$. From a technical point of view, this improvement comes from a different approach in our analysis: while the analysis in \cite{kwon2021reinforcement} partially relies on the closeness in the estimated latent parameters, our analysis does not rely on the closeness in latent parameters at all. Note that $\epsilon^{-4}$ is the best possible sample complexity for $M=2$ if our goal is to recover the model itself up to accuracy $\epsilon$. This pattern is commonly found in the literature of learning finite mixture models ({\it e.g.,} \cite{kwon2021minimax}). To avoid this higher dependence on $\epsilon$, we do not rely on any guarantees for the parameter recovery.

To elaborate more on this, our approach is a higher-order extension of the idea -- {\it uncertainty in higher-order moments} -- which was first proposed in \cite{kwon2021reinforcement}. In this approach, we estimate higher-order statistics of latent reward models, and find an empirical model that has matching moments. Then our solution is the policy optimized for this empirical model. The main difference from \cite{kwon2021reinforcement} is that they only require second order moments estimates for the case of two latent contexts $M=2$, while our approach is based upon estimation of moments of degree greater than $2$. 
However, the work in \cite{kwon2021reinforcement} heavily relies on the assumption that $M=2$, and it is not clear how to extend their analysis to general $M \ge 3$. 

In this work, we develop an alternative approach that allows us to tackle the general $M \ge 2$ case. 
Specifically, we do not relate the exact solution of higher-order polynomial equations to a near-optimality of the returned policy as in \cite{kwon2021reinforcement}. Instead, we directly bound the total variation distance of the sequence of observations which we refer as \emph{trajectory distributions}. Specifically, we bound the total variation of the trajectory distributions between the estimated and true model by the amount of mismatch in higher-order moments for all history-dependent policies. Near-optimality of the returned policy naturally follows from this result.

We establish this result by leveraging recent advancements on learning a mixture of discrete product distributions, ({\it e.g.,} \cite{feldman2008learning, chan2013learning}), and especially, from the work in \cite{chen2019beyond}. Specifically, \cite{chen2019beyond} have shown that if any two mixtures of discrete product distributions have matching moments of degree up to $d = 2M-1$, then the two distributions are statistically equivalent. While this result cannot be directly applied to show the closeness of trajectory distributions in RMMDPs, the key idea used in \cite{chen2019beyond} can resolve the core challenge that commonly arises in both problems: there are several higher-order moments that cannot be estimated from samples, which results in the unidentifability of latent models from samples. To avoid this issue, instead of bounding the statistical error from the recovery guarantee for latent models, we show that we can bound the errors in trajectory distributions using mathematical induction argument, which is the key idea used in \cite{chen2019beyond}. More detailed technical discussion can be found in Section \ref{section:algorithm}.

\subsection{Related Work}
Recent years have witnessed a substantial progress in developing efficient RL algorithms for a number of challenging tasks arising from both theory and practice ({\it e.g.,} \cite{jaksch2010near, mnih2013playing, silver2018general, kober2013reinforcement, azar2017minimax, tang2017exploration, bellemare2016unifying}). Standard framework for RL is Markov decision process (MDP), where the exact and full knowledge of the current state is provided, and no previous history affects future events. In contrast, very little is understood for partially observable systems where the exact and full knowledge of the current state is not available. Due to the vast volume of literature, we only discuss a few works that are relevant to us.

\paragraph{Solutions for general POMDPs} As a special case of POMDPs, we may consider applying the POMDP solutions for learning a near-optimal policy in RMMDPs. There is a growing body of work that focuses on the case when single or multiple-step observations from test action sequences are {\it sufficient statistics} of the environment ({\it e.g.,} \cite{boots2011closing, krishnamurthy2016pac, azizzadenesheli2016reinforcement, golowich2022learning, efroni2022provable, liu2022partially, zhan2022pac}). In such a scenario, latent model parameters can be learned up to some parameter transformations when the system is irreducible or optimistically explored. This approach have been applied to function approximation settings in some recent work under similar sufficient statistic assumptions ({\it e.g.,} \cite{cai2022reinforcement, zhan2022pac, uehara2022provably}). However, RMMDP instances do not necessarily satisfy the statistical sufficiency of test-observation sequences, and, thus, their results do not apply for learning a near-optimal policy in RMMDPs.

\paragraph{Multitask RL} RMMDP can be considered as a special case of multitask reinforcement learning in MDP environments \cite{taylor2009transfer, brunskill2013sample, liu2016pac, hallak2015contextual} with a different reward function to each task. If we are given a sufficiently long time horizon (and some separation between contexts) for an individual task to identify the context, then we can efficiently learn the latent model by clustering the trajectories. Hence, if we can learn the latent model, we can easily learn a near-optimal policy from an estimated model. However, for such condition to hold, we need very long time horizon $H \gg SA$. Unfortunately, there are many scenarios, such as dynamic web application or medical treatments \cite{hallak2015contextual, steimle2018multi}, where we have a relatively short time-horizon $H = O(1)$ for each task and thus cannot identify the latent context or the latent model. In this work, we are motivated by such scenarios and aim to design sample efficient algorithm for such scenarios where it is not possible to identify the latent context within an episode.

\paragraph{Meta reinforcement learning} In meta or transfer learning, an agent aims to learn an adaptive policy so it can adapt quickly to the chosen objective among multiple possible candidates ({\it e.g.,} \cite{duff2002optimal, taylor2009transfer, zintgraf2019varibad}). RMMDP can also be considered as a framework for meta-learning of multiple reward functions. However, existing work on meta-learning focuses on heuristic solutions and evaluations. We focus on the information-theoretic limits of the problem with provable guarantees.

\paragraph{Learning mixtures of distributions} In learning a near-optimal policy for RMMDPs, we seek for a good estimate of the latent model, which can be cast as the problem of learning a mixture of reward distributions. In particular, when the support of reward values is finite, it involves the problem of learning a mixture of product discrete distributions \cite{freund1999estimating, feldman2008learning, chan2013learning, jain2014learning, chen2019beyond}. However, it is more challenging to recover a mixture distribution in the RMMDP setting. Then, the learner gathers samples by interacting with a dynamical and stochastic environment. We develop algorithmic and analysis tools to recover a good estimate of latent models despite the limited view of samples, and show that we can learn a near-optimal policy from the estimated model.

\paragraph{Miscellaneous} While we assume that episodes start in a sequential order, in other applications such as in recommendation systems, episodes can proceed in parallel without limit on the time-horizon \cite{maillard2014latent, gentile2014online, hu2021near, kwon2022coordinated}. In such problems, the goal is to learn an optimal policy for each episode (or task) as quickly as possible exploiting the similarity between tasks. In contrast, the goal in RMMDP is to learn the optimal adaptive {\it i.e.,} history-dependent policy for a single episode with limited time horizon.

\section{Preliminaries}

We state the problem of episodic reinforcement learning problem in reward-mixing Markov decision processes (RMMDPs), originally defined in \cite{kwon2021reinforcement} as follows: 
\begin{definition}[Reward-Mixing Markov Decision Process (RMMDP)]
    \label{definition:rm_lmdp}
    An RMMDP $\mathcal{M}$ consists of a tuple $(\mS, \mA, T, \nu, \{w_m\}_{m=1}^M, \{\mu_m\}_{m=1}^M)$ with a state space $\mathcal{S}$ and action space $\mathcal{A}$ where $T: \mathcal{S}\times\mathcal{A}\times\mathcal{S} \rightarrow [0,1]$ is a common transition probability measures that maps a state-action pair and a next state to a probability, and $\nu$ is a common initial state distribution. Let $S = |\mS|$ and $A = |\mA|$; $w_1, ..., w_M$ are the mixing weights such that at the beginning of every episode a reward model $\mu_m$ is randomly chosen with probability $w_m$; $\mu_m$ is the \emph{model parameter} that describes a reward distribution, {\it i.e.,} $\PP_{\mu_m} (r \mid a) := \PP(r \mid m, a)$, according to an action $a \in \mA$ conditioning on a latent context $m$. 
\end{definition}
We do not assume a priori knowledge of mixing weights. We consider discrete reward realizations, when the support of the reward distribution is finite and bounded. 
\begin{assumption}[Discrete Rewards]
\label{assumption:reward_dist}
    The reward distribution has finite and bounded support. The reward attains a value in the set $\mathcal{Z}$. We assume that for all $z\in \mathcal{Z}$ we have $|z| \le 1$. We denote the cardinality of $ \mathcal{Z}$ as~$Z$. 
\end{assumption}
As an example, Bernoulli distribution satisfies Assumption \ref{assumption:reward_dist} with $\mZ=\{0,1\}$ and $Z = 2$.  We denote the probability of observing a reward value $z$ by playing an action $a$ at a state $s$, as $\mu_m(s,a,z) \coloneqq \PP(r = z \mid m, s, a)$ in a context~$m$. 
We consider a policy class $\Pi$ which contains all history-dependent policies $\pi: (\mS, \mA, \mZ)^* \times \mS \rightarrow \mA$. We are interested in finding a near-optimal policy $\pi \in \Pi$ that is $\epsilon$-optimal w.r.t. the optimal value:  
\begin{align*}
    V_{\mM}^* := \max_{\pi \in \Pi} \Exs_\pi \left[ \tssum_{t=1}^H r_t \right],
\end{align*}
where $\Exs_\pi [\cdot]$ is expectation taken over the model $\mM$ with a policy $\pi$.

\paragraph{Notation}
We often denote a state-action pair $(s,a)$ as one symbol $x = (s,a) \in \mS \times \mA$. For any length $t$ part of a trajectory $(y_1, y_2, ..., y_t)$, we often simplify the notation as $y_{1:t}$ for any symbol $y$. For any pairs of length $d$ sequences $\bm{x} = (x_i)_{i=1}^d$ and $\bm{z} = (z_i)_{i=1}^d$ (sequences may have repeated elements), let ${\rm length}(\bm{x})$ be the length of sequence. For a subset of indices $\mI \subseteq [d]$, we use $\bm{x}_{\mI} := (x_{i})_{i \in \mI}$ to refer to a subsequence in $\bm{x}$ at positions $\mI$. A moment of degree $d$ for $\bm{x}$  and $\bm{z}$ is defined as
\begin{align*}
    \moment \left(\bm{x}, \bm{z} \right) :=  \textstyle \sum_{m=1}^M w_m \Pi_{i \in [{\rm length}(\bm{x})]} \mu_m(x_i, z_i).
\end{align*}
Let $n(\bm{x})$ be the number of martingale samples used to estimate $\moment \left(\bm{x}, \bm{z} \right)$ for all $z \in \mZ^{\bigotimes {\rm length}(\bm{x})}$. We denote $V_\mM^\pi$ as an expected cumulative reward for model $\mM$ with policy $\pi$. \YE{defining O tilde notation.} 

\section{Learning from Uncertain Higher-Order Moments}
\label{section:algorithm}

The basic idea for learning a near-optimal policy of an RMMDP has been developed in \cite{kwon2021reinforcement} for the special case of $M=2$ with equal mixing weights. There, they showed that for $M=2$, having access to estimates of the second-order correlation of rewards is sufficient to find a near-optimal policy. In this section, we revisit this idea and extend it to handle a general number of contexts $M\ge 2$. 

\subsection{Overview of Our Approach}
\begin{algorithm}[t]
    \caption{Estimate and Match Moments (\texttt{EM$^2$})}
    \label{algo:reward_mixture}
        
    {{\bf Input:} $d \in \mathbb{N}_+$, $\epsilon,\eta \in(0,1)$}
    \begin{algorithmic}[1]
    \STATE {\color{blue}{// Estimate moments of latent reward and their uncertainty by pure exploration of the $d$-th order MDP}}
    \STATE {$(\moment_n(\cdot, \cdot), n(\cdot)) \gets \texttt{EstimateMoments}(d,\epsilon, \eta)$ (see Appendix \ref{appendix:pure_exploration_algo})}
        \STATE {\color{blue}{// Construct an RMMDP with matching moments}}
        \STATE {Find $\hat{\mM}$ such that  $|\hat{\moment}(\bm{x}, \bm{z}) - \moment_n(\bm{x}, \bm{z})| \le \sqrt{\iota_c / n(\bm{x})}$ for all $\bm{x} \in (\mS\times\mA)^{\bigotimes d}, \bm{z} \in \mZ^{\bigotimes d}$}
        \STATE {{\bf Return} $\hat{\pi}$,  optimal policy of $\hat{\mM}$}
    \end{algorithmic}
\end{algorithm}

Our approach is based upon estimating higher-order correlations of the reward model through, which later enables us to access an approximate RMMDP. Then, we return the optimal policy of the learned RMMDP model. As discussed in \cite{kwon2021reinforcement}, if we can measure the exact values of $\moment(\bm{x}, \bm{z})$ for all $\bm{x} = (x_i)_{i=1}^{q} \in (\mS\times\mA)^{\bigotimes q}$ and $\bm{z} = (z_i)_{i=1}^q \in \mZ^{\bigotimes q}$ up to some large-enough degree $q \le d \in \mathbb{N}_+$, then we can recover the latent model $\{(w_m, \mu_m)\}_{m=1}^M$. This is the well-known moment-matching technique in literature on learning finite mixture models (see {\it e.g.,} \cite{moitra2010settling, doss2020optimal} and references therein). 

However, the challenge of applying moment-matching technique to learning in RMMDP is that the samples from RMMDP is based on trajectories obtained from roll-in policies, and thus it is in general not possible to accurately estimate every higher-order statistics. For example, in a loop-free system, any state-action $x$ cannot be visited more than once in the same episode, in which case we cannot get any samples of the higher-order moment that repeats the same state more than once ({\it i.e.,} the moments $\moment(\bm{x}, \bm{z})$ with multiplicity that has repeated elements in $\bm{x} = (x_i)_{i=1}^d$). In such a case, the true latent reward model is not identifiable. 

As mentioned earlier, this {\it model unidentifiability} issue can also be found in -- seemingly unrelated -- literature of learning mixtures of discrete product distributions \cite{freund1999estimating, feldman2008learning, chen2019beyond}. There, the task of learning latent mixture parameters is also challenging due to the model unidentifiability issue, since higher-order statistics with multiplicity cannot be estimated. Thus, most work in this line focused on the density estimation which minimizes the {\it statistical distance} between observations, rather than insisting on recovering latent parameters (there are a few exceptions, {\it e.g.,} \cite{gordon2021source}). Although our problem is in a different context of learning a near-optimal policy, we can cast a similar fundamental question:
\begin{center}
    \emph{Can we find a model $\hat{\mM}$ such that for every policy $\pi \in \Pi$, trajectory distributions from the true $\mM$ and estimated model $\hat{\mM}$ are statistically close?}
\end{center}
In other words, we ask whether the exact model recovery is really necessary when our ultimate goal is just find a good-working policy. If we can find such a model $\hat{\mM}$ that approximates the {\it trajectory distributions} regardless of the unidentifiability issue, then it can be used to find a near-optimal policy for the true model $\mM$ (see more details in Section \ref{section:analysis}).

Then the question is how to find such a model $\hat{\mM}$ that approximates the trajectory distributions for {\it all} policies. Note that if we only need to approximate a distribution for a {\it single} policy $\pi$, then the problem can be addressed using the classical tournament argument \cite{devroye2001combinatorial} with polynomial sample-complexity. However, learning for a single policy is not enough to learn a near-optimal policy, and there are doubly-exponential number of candidate policies $O \left(A^{(SA)^{H}} \right)$ in the class of all history-dependent policies. We need more thoughts to design sample-efficient algorithm.

It turns out that the idea of learning {\it uncertain higher-order moments}, initially proposed in \cite{kwon2021reinforcement}, can exactly achieve this goal with general $M \ge 2$. Specifically, for any length $d\ge1$ state-action sequences $\bm{x} = (x_i)_{i=1}^d$, let $p(\bm{x})$:
\begin{align}
    p(\bm{x}) := \textstyle \max_{\pi \in \Pi} \PP_{\pi} (x_{1:H} \text{ contains a (not necessarily time-consecutive) subsequence } \bm{x} ), \label{eq:p(x) defintion}
\end{align}
the maximum probability of visiting all elements in $\bm{x}$ in the same episode. Intuitively, the larger $p(\bm{x})$ is, the more accurate estimate of $\moment(\bm{x}, \cdot)$ is required. That is, we collect correlation samples for $\bm{x}$ such that the number of collected samples $n(\bm{x})$ is roughly proportional to $p(\bm{x})$. In Section \ref{section:analysis}, we show that if we can explore the environment to collect samples of higher-order moments in such a way, then trajectories distributions of {\it all} history-dependent policies are {\it uniformly} close the true model. The learning procedure can be summarized as in Algorithm \ref{algo:reward_mixture}.
\begin{remark}[Unknown $T, \nu$]
    In the first read, readers may assume that the transition kernel and initial state distribution $T, \nu$ are known in advance, and given as input to moment estimation procedure. When they are unknown, we can obtain good estimates $\hat{T}, \hat{\nu}$ along with higher-order moments of latent reward models in the exploration phase. See the full Algorithm \ref{algo:pure_explore} in Appendix \ref{appendix:pure_exploration_algo}. 
\end{remark}

\subsection{Pure Exploration of Higher-Order Moments}
The estimation of $\moment(\bm{x}, \cdot)$ can be carried in multiple ways. A simple approach for doing that is to iterate over all moments up to degree $d$, {\it i.e.,} for all $\bm{x} \in \cup_{q=1}^d (\mS \times \mA)^{\bigotimes q}$, and run the best policy for collecting trajectories that contains $\bm{x}$. However, this approach will waste many trajectories for collecting samples of moments that are hard to reach, resulting in total sample-complexity of $O(SA)^{2d}$. 

Note that the aim of pure-exploration is to collect samples for all moments up to degree $d$ such that $n(\bm{x}) \propto p(\bm{x})$. A simple but more systematic way for estimating $\moment(\bm{x}, \cdot)$ is to employ a pure exploration scheme \cite{kaufmann2021adaptive}, analogously to the idea developed in \cite{kwon2021reinforcement} for $M=2$. For completeness, we restate some concepts and definitions for pure-exploration in higher-order MDPs in Appendix \ref{appendix:pure_exploration_algo}. After carrying out a generalized procedure, we are guaranteed to have good estimates of higher-order moments for any $d$. \jycomment{mention $d^{th}$ order MDP?} 

\subsection{Find a Moment-Matching Model}
Once the pure exploration phase ends, we have a collection of samples for all moments of degree at most~$d$. Specifically, for any degree $q \le d$, moment $\bm{x} \in (\mS \times \mA)^{\bigotimes q}$ with any paired sequence $\bm{z} \in \mZ^{\bigotimes q}$, let the quantity $\moment_n (\bm{x}, \bm{z})$ be an empirical estimate of $\moment (\bm{x}, \bm{z})$ (see Algorithm \ref{algo:pure_explore} for more details). By a standard measure of concentration for martingales \cite{wainwright2019high}, we can show that
\begin{align*}
    \left| \moment (\bm{x}, \bm{z}) - \moment_n (\bm{x}, \bm{z}) \right| \le \sqrt{\iota_c / n(\bm{x})}.
\end{align*}
The above holds for over all combinations of $\bm{x}$ and $\bm{z}$ with probability at least $1 - \eta$ by an application of union bound, where the logarithmic constant $\iota_c = O(d \log(SAZ / \eta))$. With $\moment_n (\cdot, \cdot)$, we search over all RMMDP models to find an empirical model $\hat{\mM}:= (\mS, \mA, \hat{T}, \hat{\nu}, \{\hat{w}_m\}_{m=1}^M, \{\hat{\mu}_m\}_{m=1}^M)$ that satisfies $\left| \hat{\moment} (\bm{x}, \bm{z}) - \moment_n (\bm{x}, \bm{z}) \right| \le \sqrt{\iota_c / n(\bm{x})}$, then we are guaranteed that 
\begin{align}
    \left| \moment (\bm{x}, \bm{z}) - \hat{\moment} (\bm{x}, \bm{z}) \right| \le 2 \sqrt{\iota_c / n(\bm{x})}, \qquad \forall (\bm{x}, \bm{z}) \in \textstyle \bigcup_{q=1}^d (\mS \times \mA)^{\bigotimes q} \times \mZ^{\bigotimes q}. \label{eq:moment_uncertainty} 
\end{align}
That is, we find an RMMDP model where its first $d$ moments approximately match the ones of the true model.

\paragraph{Computational Challenges for the Model Recovery}
Solving the equation \eqref{eq:moment_uncertainty} is not an easy computational task. Brute-force approaches, which iterates over all possible candidates, may take time exponential in $O(SA)$. Even for a simpler setting of learning mixtures of discrete $n$ product distributions, it is not obvious to find the latent parameters that matches all multilinear moments ({\it i.e.,} moments without any multiplicity) \cite{feldman2008learning, chen2019beyond}. The best known computational complexity for that problem is $O(n/\epsilon)^{O(M^2)}$ due to \cite{chen2019beyond}. Since we have more non-uniform uncertainties across all moments, we expect that solving \eqref{eq:moment_uncertainty} is computationally much harder problem. We leave the computational challenge to future work, and henceforth focus on the sample-complexity upper bound of learning near optimal policy of RMMDP. 

\section{Main Theoretical Results}
\label{section:analysis}
To simplify the discussion, we momentarily assume that transition models are known, {\it i.e.,} $T$ and $\nu$ are given. This section focuses on analyzing the performance difference between two RMMDP models $\mM^{(1)}$ and $\mM^{(2)}$, where $\mM^{(1)}$ is the true RMMDP model, and $\mM^{(2)}$ is an empirical RMMDP model who has the same transition and initial state probabilities $T$ and $\nu$, but different latent reward model and mixing weights, {\it i.e.,} $T^{(2)} = T$, $\nu^{(2)} = \nu$, and $w_m^{(2)} = \hat{w}_m, \mu_m^{(2)} = \hat{\mu}_m$. We note that prior knowledge of $T$ and $\nu$ is not required in our final result.

As mentioned earlier, for any fixed policy $\pi \in \Pi$, difference in expected rewards can be bounded by $l_1$-statistical distance in trajectory distributions. More specifically, consider the set of all possible trajectories $\mT = (\mS \times \mA \times \mZ)^{\bigotimes H}$, that is, any state-action-reward sequence of length $H$. Then, 
\begin{align}
    |V_{\mM^{(1)}}^{\pi} - V_{\mM^{(2)}}^{\pi}| \le H \cdot \|(\PP^{(1)}_{\pi} - \PP^{(2)}_{\pi}) ((x,r)_{1:H})\|_1 = H \cdot \tssum_{\tau \in \mT} |\PP^{(1)}_{\pi} (\tau) - \PP^{(2)}_{\pi} (\tau)|, \label{eq:value_bound_from_distributions}
\end{align}
For any policy $\pi \in \Pi$, our goal is to show that ${\sum_{\tau \in \mT} |\PP_\pi^{(1)} (\tau) - \PP_{\pi}^{(2)} (\tau)| \le O(\epsilon / H)}$, {\it i.e.,} the true and empirical models are close in $l_1$-statistical distance for all history-dependent policies.

The main challenge in the analysis is to bound the $l_1$ distance of length $H$ trajectories without exponential dependence on $H$. We need to show that the trajectory distribution of {\it any} policy is close in total variation distance if $O(\min(M,H))$-higher order reward moments of $\mM^{(1)}$ and $\mM^{(2)}$ match. 
Such result is established in \cite{kwon2021reinforcement}, for $M=2$, by specifying the conditions under which some parts of the reward model $\mu_m(x,z)$ are partially identifiable, and further showing that all non-identifiable do not significantly affect on the quality of the returned policy.

However when $M \ge 3$, sufficient conditions even for partial identifiability are often too strong to hold~\cite{gordon2021source}, and is not guaranteed in general. In this work, we side-step the unidentifiability challenge by only analyzing the uncertainty of higher-order moments. This also implies that model identifability is not necessary for learning a near-optimal policy.

We now focus on this issue in more detail. Assume we have non-uniform confidence levels for different moments of  $\moment (\bm{x}, \bm{z})$ which corresponds to different values of $n(\bm{x})$. Similarly to the technique developed in~\cite{kwon2021reinforcement}, we divide the level of uncertainties of trajectories based on the number of samples collected for each moment. We define the following sets: 
\begin{align}
    \label{eq:define_countlevel_sets}
    \mX_l &= \{ \bm{x} \in \textstyle \bigcup_{q=1}^d (\mS\times \mA)^{\bigotimes q} \ | \ n (\bm{x}) \ge n_l\}, \nonumber \\
    \Eps_l &= \{ x_{1:H} \in (\mS \times \mA)^{\bigotimes H} \ | \ \forall q \le d: \forall \ 1 \le t_1 < \ldots < t_q \le H, \ (x_{t_i})_{i=1}^q \in \mX_{l} \},
\end{align}
where $\{n_l\}_{l=0}^{L}$ is a geometrically decreasing threshold sequence for sufficient exploration of moments at all levels $l \in \{0, 1, \ldots, L\}$, and $L$ is the largest integer such that the moments explored below $n_{L}$ times are considered as non-reachable pairs. We specify the values of $(n_l)_{l=1}^l$ in Lemma \ref{lemma:reward_free_guarantee}. Here, $\Eps_l$ is a set of length at most $d$ state-actions in which every subsequence has been sampled at least $n_l$ times. Furthermore, observe that $\Eps_0 \subseteq \Eps_1 \subseteq \cdots \subseteq \Eps_L$. Then we split a set of trajectories into disjoint sets $\Eps_0' = \Eps_0$, $\Eps_{L+1}' = \Eps_L^{\rm c}$ and $\Eps_l' = \Eps_{l-1}^\mathrm{c} \cap \Eps_l$ and for $l \in [L]$, {\it i.e.,} a set of trajectories with all correlations of degree at most $d$ sampled more than $n_l$ and at least one set of correlation explored less than $n_{l-1}$ times. 

Recall that our goal is to control the $l_1$-statistical distance between distributions of trajectories for any history-dependent policies that resulted from true and empirical models. With the above machinery, we can, instead, bound the $l_1$ statistical distance between all trajectories {\it for all policies}. As in \cite{kwon2021reinforcement}, we aim to bound the statistical distance separately by level: 
\begin{align}
    \label{eq:bound_total_variation_distance_overall}
    \|\PP^{(1)}_\pi - \PP^{(2)}_\pi\|_1 = \sum_{l=0}^{L+1} \sum_{\tau:x_{1:H} \in \Eps_l'} |\PP^{(1)}_{\pi}(\tau) - \PP^{(2)}_{\pi}(\tau) | \le \sum_{l = 0}^{L+1} \sup_{\pi \in \Pi} \PP^{(1)}_\pi( x_{1:H} \in \Eps_l' ) \cdot O(\epsilon_l),
\end{align}
where $\PP_\pi^{(1)} ( x_{1:H} \in \Eps_l'),$ is the probability that a random trajectory $\tau$ observed with a roll-in policy $\pi$ belongs to $\Eps_{l}'$, and $\epsilon_l$ is the overall statistical distance of trajectory distributions (conditioned on $\Eps_{l}'$) at level $l$. Further, observe that the first relation in equation~\eqref{eq:bound_total_variation_distance_overall} holds since any trajectory $\tau$ belongs to one of the sets $\Eps_l'$. The second relation of equation~\eqref{eq:bound_total_variation_distance_overall} holds by Lemma \ref{lemma:eventwise_total_bound}. 
The main challenge of the analysis is to give a proper upper bound of $\epsilon_l$, and we will focus on relating $\epsilon_l$ with the closeness in reachable moments.

Specifically, let $\delta_{l}$ be the maximum error between the $d^{th}$ order reachable moments in level $l$:
\begin{align}
    \delta_{l} := \max_{\bm{x} \in \mX_l} \max_{\bm{z} \in \mZ^{\bigotimes d} } \max_{\mI \subseteq [d]} \left| \moment^{(1)}(\bm{x}_{\mI}, \bm{z}_{\mI}) - \moment^{(2)} (\bm{x}_{\mI}, \bm{z}_{\mI}) \right|.  \label{eq:delta_def}
\end{align}
As seen from equation \eqref{eq:moment_uncertainty}, we have $\delta_l \propto \sqrt{\iota_c/n_l}$ if we think about $\mM^{(2)}$ as an empirical approximation of $\mM^{(1)}$. Given a bounded mismatch in moments of degree up to $d$, the technical crux of the analysis is the robust approximation of trajectory distributions. We prove the following lemma on distributional closeness in trajectories that belong to $\Eps_l$ for all policies:
\begin{lemma}[Eventwise Total Variance Discrepancy]
    \label{lemma:eventwise_total_bound}
    Let $\delta_l$ be defined as in equation~\eqref{eq:delta_def}, {\it i.e.,} the maximum mismatch in moments up to degree $d = \min(2M-1, H)$. For any $l \in \{0, 1, ..., L \}$ and any history-dependent policy $\pi \in \Pi$, we have:
    \begin{align}
        \label{eq:eventwise_total_variance_diff}
        \sum_{\tau: x_{1:H} \in \Eps_l'} | \PP^{(1)}_\pi (\tau) - \PP^{(2)}_\pi(\tau) | \le \sup_{\pi \in \Pi} \PP^{(1)}_\pi (x_{1:H} \in \Eps_l') \cdot (4HZ)^{\min(2M-1,H)} \cdot \delta_l. 
    \end{align}
\end{lemma}
Lemma \ref{lemma:eventwise_total_bound} implies the second relation in \eqref{eq:bound_total_variation_distance_overall} when setting $\epsilon_l = (4HZ)^{\min(2M-1, H)} \cdot \delta_l$. Further, it also generalizes an analogous result that was proved in~\cite{kwon2021reinforcement} for the $M=2$ case (see ~\cite{kwon2021reinforcement}, Lemma 4.1). There are several notable differences between these results:
\begin{enumerate}
    \item  The threshold value $n_l$ is set to the order of $\delta_l^{-2}$. In contrast, in \cite{kwon2021reinforcement}, $n_l$ was set to be of the order of $\delta_l^{-4}$. This improvement  leads to the optimal dependence in $\epsilon$ in our final result, namely, $O(\epsilon^{-2})$ instead of~$O(\epsilon^{-4})$.
    \item We do not rely on any parameter recovery guarantee which is also hard to get with a general number of contexts $M\ge 3$. Instead, we directly convert the closeness in moments to closeness in total variation distance of the trajectory distributions for all history-dependent policies. We prove this result by using a mathematical induction on the number of contexts and time-horizon. This is inspired by and adopted from \cite{chen2019beyond} where a similar proof idea is used for showing the robust identifiability of mixtures of discrete product distributions directly from closeness in moments (see Lemma 5.5 in~\cite{chen2019beyond}). 
\end{enumerate}
To elaborate more on the second point, by mathematical induction we aim to prove a slightly more general result which can be of independent interest:
\begin{theorem}[Bound on Total Variation from Moment Closeness]
    \label{theorem:distance_from_moment_theorem}
    Let $\delta>0$ and let $\mathcal{M}^{(1)}, \mathcal{M}^{(2)}$ be two RMMDPs. Assume that  $\mathcal{M}^{(1)}$ and $\mathcal{M}^{(2)}$ have similar transition kernel and initial state distribution, but have different latent reward models, and $M_1$ and $M_2$ number of latent contexts, respectively. Define $\mX_d$ to be the set of length $d := \min(H, M_1 + M_2 - 1)$ state-action sequences that have nearly matched moments 
    \begin{align*}
        \mX_d := \left\{\bm{x} \in (\mS\times\mA)^{\bigotimes d} \Big| \forall \bm{z} \in \mZ^{\bigotimes d} : \ \max_{\mathcal{I} \subseteq [d]} \left| \moment^{(1)} (\bm{x}_{\mI}, \bm{z}_{\mI}) - \moment^{(2)} (\bm{x}_{\mI}, \bm{z}_{\mI}) \right| \le \delta \right\}. 
    \end{align*}
    Let $\Eps_{\rm tot}$ be the set of trajectories for which all subsequences of length $d$ are in $\mX_d$, {\it i.e.,} $\Eps_{\rm tot}$ is the set of all well-explored trajectories:
    \begin{align*}
        \Eps_{\rm tot} := \left\{ x_{1:H} \in (\mS \times \mA)^{\bigotimes H} | \forall \ 1 \le t_1 < \ldots < t_{d} \le H: \ (x_{t_q})_{q=1}^{d} \in \mX_d \right\}. 
    \end{align*}
    Then for any subset of well-explored trajectories $\Eps \subseteq \Eps_{\rm tot}$, for any history-dependent policy $\pi$, we have 
    \begin{align*}
        \sum_{\tau: x_{1:H} \in \Eps} | \PP^{(1)}_{\pi} (\tau) - \PP^{(2)}_{\pi} (\tau) | 
        \le \sup_{\pi \in \Pi} \PP^{(1)}_{\pi} (x_{1:H} \in \Eps) \cdot (4HZ)^d \cdot \delta.
    \end{align*}
\end{theorem}
That is, instead of bounding total variation distance with the same number of latent contexts, we allow the two models to have difference number of contexts as long as the moments can be matched. Then, we can apply the mathematical induction steps by reducing the total number of contexts $M_1+M_2$ by one at each step. Then, Lemma \ref{lemma:eventwise_total_bound} is a direct corollary of Theorem \ref{theorem:distance_from_moment_theorem}. We refer the readers to Appendix \ref{appendix:proof_main_lemma} for the further details.

Once we convert the moment closeness to statistical closeness, then the second key connection can be made between $\sup_{\pi \in \Pi} \PP^{(1)}_{\pi} (\Eps_l')$ and the data collected in pure-exploration phase. The following lemma relates them:
\begin{lemma}[Higher-Order Version of Lemma 4.2 in \cite{kwon2021reinforcement}]
    \label{lemma:reward_free_guarantee}
    There exists a pure-exploration algorithm which takes $\epsilon_{\rm pe} > 0$ as an input parameter, such that with probability at least $1-\eta$, if it explores higher-order moments using at most $K$ episodes where
    \begin{align}
        \label{eq:required_episodes}
        K \ge C \cdot (SA)^{d} \epsilon_{\rm pe}^{-2}  \log(K/ \eta),
    \end{align}
    with some absolute constant $C > 0$, and with equation \eqref{eq:define_countlevel_sets} where we set $n_0 = K/(SA)^d, n_{l+1} = n_l/4$ for $l = 0, 1, \ldots, L$, and $L$ such that $n_L > \iota_c$ and $n_{L+1} \le \iota_c$, then for all $l \in \{0, 1, \ldots, L+1\}$ we have
    \begin{align}
        \sup_{\pi \in \Pi} \PP_\pi (x_{1:H} \in \Eps_{l}') \le O \left(H^{d} \epsilon_{\rm pe} \cdot \sqrt{n_l / \iota_c} \right). \label{eq:max_event_probability}
    \end{align}
\end{lemma}
Proof of Lemma \ref{lemma:reward_free_guarantee} is given in Appendix \ref{appendix:reward_free_guarantee}. We set $\epsilon_{\rm pe} = \epsilon / \left(HL (4H^2Z)^d \right)$, $L = O(\log(n_0)) \le O(\log K)$ and $d = \min(2M-1,H)$. We apply Lemma \ref{lemma:reward_free_guarantee} and \eqref{eq:bound_total_variation_distance_overall} to bound a difference in expected value of true and empirical models for an arbitrary history-dependent policy $\pi$:
\begin{align*}
    |V_{\mM^{(1)}}^\pi - V_{\mM^{(2)}}^\pi| \le H \cdot \|(\PP_\pi^{(1)} - \PP^{(2)}_\pi) (\tau)\|_1 \le H \tssum_{l = 0}^{L} H^{d} \cdot O(\epsilon_{\rm pe}) \cdot (4HZ)^d.
\end{align*}
Our main theorem combines Lemma \ref{lemma:eventwise_total_bound} and \ref{lemma:reward_free_guarantee}, giving our main sample complexity result:
\begin{theorem}[Sample Complexity of Learning RMMDPs with $M\ge2$]
    \label{theorem:sample_complexity_upper_bound}
    Let $d = \min(2M-1, H)$. There exists a universal constant $C > 0$ such that there exists an algorithm using at most $K$ episodes where,
    \begin{align*}
        K \ge C \cdot \frac{(SA)^{d}}{\epsilon^2} \cdot \poly(d,H,Z)^{d} \cdot \poly \log (K/\eta),
    \end{align*}
    and outputs an $\epsilon$-optimal policy with probability at least $1 - \eta$. 
\end{theorem}

\subsection{Improved Results for $M=2$}
The work of~\cite{kwon2021reinforcement} in which the problem of learning RMMDP for $M = 2$ was first studied analyzed the case in which the mixing weights are balanced, {\it i.e.,} $w_1 = w_2 = 1/2$. There, the authors designed an algorithm with sample complexity of $O(poly(H,Z)(SA)^2/\epsilon^4)$ without further assumptions. We now show that for the special setting considered in~\cite{kwon2021reinforcement} Theorem~\ref{theorem:sample_complexity_upper_bound} can be improved to yield an upper bound of $O(poly(H,Z)(SA)^2/\epsilon^2)$: strictly improving the dependence of $\epsilon$ without resulting in degradation in the polynomial dependence of $(SA)$.

The following lemma is key to the improved result.
\begin{lemma}
    \label{lemma:third_moment_matching_m2}
    For any RMMDP with $M = 2$ and $w_1 = w_2 = 1/2$, the following holds: for any length three sequences of state-action $\bm{x} = (x_i)_{i=1}^3$ and rewards $\bm{z} = (z_i)_{i=1}^3$, 
    \begin{align*}
        \moment(\bm{x}, \bm{z}) &=  -2 \moment(\bm{x}_{\{1\}}, \bm{z}_{\{1\}}) \cdot \moment(\bm{x}_{\{2\}}, \bm{z}_{\{2\}}) \cdot \moment(\bm{x}_{\{3\}}, \bm{z}_{\{3\}}) + \moment(\bm{x}_{\{1\}}, \bm{z}_{\{1\}}) \cdot \moment(\bm{x}_{\{2,3\}}, \bm{z}_{\{2,3\}}) \\
        &\quad + \moment(\bm{x}_{\{2\}}, \bm{z}_{\{2\}}) \cdot \moment(\bm{x}_{\{1,3\}}, \bm{z}_{\{1,3\}})
        + \moment(\bm{x}_{\{3\}}, \bm{z}_{\{3\}}) \cdot \moment(\bm{x}_{\{1,2\}}, \bm{z}_{\{1,2\}}).
    \end{align*}
\end{lemma}
That is, for this special case, if the first and second moments nearly match, then the third moments are also guaranteed to match. Equipped with the above lemma along with Theorem \ref{theorem:sample_complexity_upper_bound}, we can get a corollary that strictly improves the result of \cite{kwon2021reinforcement}:
\begin{corollary}[Improved Sample Complexity for Balanced 2-RMMDPs]
    \label{corollary:sample_upper_bound_m2}
    There exists a universal constant $C > 0$ such that if $M=2$ and $w_1=w_2=1/2$, then there exists an algorithm using at most $K$ episodes where,
    \begin{align*}
        K \ge C \cdot \frac{(SA)^{2}}{\epsilon^2} \cdot \poly(H,Z) \cdot \poly \log(K/\eta),
    \end{align*}
    such that outputs an $\epsilon$-optimal policy with probability at least $1 - \eta$. 
\end{corollary}
We believe the idea of expressing the third-order moment using lower-order moments can also be applied when the prior is unknown with extra exploration procedures (see, {\it e.g.}, Appendix E in \cite{kwon2021reinforcement}). We leave this as future work. 

\subsection{$(SA)^{O(\log M)}$-Upper Bound with Integral Probabilities}
We have shown that for general instances of RMMDPs, we can learn an $\epsilon$-optimal policy using $O(SA)^{2M-1}$ samples. This sample complexity can be significantly improved if we have an additional assumption on latent reward models. Suppose that for any $m \in [M], x \in \mS \times \mA, z \in \mZ$, $\mu_m(x,z)$ can take a value from a finitely discretized set $\mP = \{0, 1/P, \ldots, 1\}$ for some positive integer $P \in \mathbb{N}_+$. For integral reward probabilities, we show that we only need to match up to degree $d = \lceil 2P\log M \rceil$ moments, and thus $(SA)^{O(\log M)} / \epsilon^2$ samples are sufficient to learn an $\epsilon$-optimal policy. One interesting special case of such scenario is when a reward is deterministic conditioned on a context, {\it i.e.,} $\mu_m(x,z)$ takes value from $\mP = \{0,1\}$ with $P = 1$. 

This is a reminiscent of quasi-polynomial sample-complexity for learning a mixture of subcubes \cite{chen2019beyond}, {\it i.e.,} learning a mixture of binary product distributions $\mZ = \{0,1\}$ when the latent model parameters can only take values from $\mP = \{0,1/2,1\}$. While not used for a more general setting, we show that their main identifiability (of distribution from moments) results can be similarly applied to RMMDP problems with general observation support $\mZ$ and integral probability set $\mP$. 
\begin{lemma}[Modified Lemma \ref{lemma:eventwise_total_bound} for Integral Probabilities]
    \label{lemma:eventwise_total_bound_discrete}
    Suppose $\mu_m(x,z)$ takes values only from $\mP = \{ 0, 1/P, \ldots, 1\}$ for all $m, x, z$. Let $\delta_l$ is defined as in \eqref{eq:delta_def} for the maximum mismatch in moments up to degree $d = \min(\lceil 2P\log M \rceil, H)$. For any $l \in \{0, 1, ..., L+1 \}$ and any history-dependent policy $\pi \in \Pi$, we have
    \begin{align}
        \label{eq:eventwise_total_variance_diff}
        \sum_{\tau: x_{1:H} \in \Eps_l'} | \PP^{(1)}_\pi (\tau) - \PP^{(2)}_\pi(\tau) | \le \sup_{\pi \in \Pi} \PP^{(1)}_\pi (x_{1:H} \in \Eps_l') \cdot M^{O(M P \log P)} \cdot \delta_l. 
    \end{align}
\end{lemma}
See Appendix \ref{appendix:proof_main_lemma_discrete_rewards} for the proof. Combining Lemma \ref{lemma:eventwise_total_bound_discrete} with Lemma \ref{lemma:reward_free_guarantee}, we get the following quasi-polynomial sample-complexity result for integral reward probabilities:
\begin{theorem}
    \label{theorem:sample_complexity_upper_bound_discrete}
    Suppose $\mu_m(x,z)$ takes values only from $\mP = \{ 0, 1/P, \ldots, 1\}$ for all $m, x, z$ where $P \in \mathbb{N}_+$ is an absolute constant. If $H > 2P\log M$, then there exists a universal constant $C > 0$ such that there exists an algorithm using at most $K$ episodes where,
    \begin{align*}
        K \ge C \cdot \frac{(SA)^{2P\log M}}{\epsilon^2} \cdot M^{O(M)} \cdot \poly(H) \cdot \poly \log (K/\eta),
    \end{align*}
    and outputs an $\epsilon$-optimal policy with probability at least $1 - \eta$. 
\end{theorem}
Note that for the case of deterministic rewards, we can apply the above theorem with $P = 1$.

\section{Lower Bound}
\label{section:lower_bound}

In previous sections, we designed an algorithm that learns a near-optimal policy for general instances of RMMDPs with discrete rewards given $O\left((SA)^{O(M)}\right)$ samples. In this section, we complement this upper bound by showing that a super polynomial dependence on $S$ and $A$ is necessary for $M = \omega(1)$ from information-theoretic standpoint. Specifically, we show that there exists a class of instances which cannot avoid $(SA)^{\Omega(\sqrt{M})}$ sample complexity. 

To show this lower bound, we construct the hard instance $\mM$ as follows: at every time step $t \in [H]$, we deterministicially move to a unique state $s_t^*$, and the reward values are binary, {\it i.e.,} $\mZ = \{0,1\}$. At every state $s_t = s_t^*$ (or time step $t$), all actions except one {\it correct} action $a_t^* \in \mA$ returns a reward sampled from a uniform distribution over $\{0,1\}$. In this section, since we only consider binary rewards, we omit the $z$-part for indexing $\mu_m$ with $(x,z)$, {\it i.e.,} use $\mu_m(x)$ to mean $\mu_m(x,1) = \Exs[\indic{r = 1} | x]$. We also use $\moment(\bm{x}) := \sum_{m=1}^M w_m \Pi_{i=1}^{|\bm{x}|} \mu(x_i)$ for the moments of degree up to $H$ for every state-action sequence $\bm{x}$.

We want to construct an example such that for all but the correct sequence of actions $a_{1:H} = a^*_{1:H}$, distributions of observed reward sequences are not statistically distinguishable from playing uniform actions. Such an example can be constructed by finding a moment-matching correct actions. Specifically, let $d = H = \Omega(\sqrt{M})$ be the desired degree of matching moments that we need for the construction of hard instances. For simplicity, let $\mu_{m}^* \in \mathbb{R}^d$ be the restriction of $\mu_m$ to correct actions, {\it i.e.,} $\mu_m^*(t) = \mu_m(s_t^*, a_t^*)$ for all $t \in [d]$. The desired hard instance can be found in \cite{chen2019beyond}, which proves the following lemma:
\begin{lemma}[Result of Section 4.3 in \cite{chen2019beyond}]
    \label{lemma:moment_matching_lower_bound}
    There exists some $d = \Omega(\sqrt{M})$ such that for any $\epsilon \le (2d)^{-2d}$, there exists a realization $\{\mu_m^*\}_{m=1}^M$ and mixing weights $\{w_m\}_{m=1}^M$, such that all degree $q < d$ multilinear moments of $\mu_m^*$ is equal to $(1/2)^q$:
    \begin{align*}
        \textstyle \sum_{m=1}^M w_m \Pi_{t \in \mI} \mu_m^*(t) = (1/2)^{q}, \qquad \forall \mI \subsetneq [d]: |\mI| = q.  
    \end{align*}
    Furthermore, the degree-$d$ moment is $\epsilon$-away from the uniform distribution:
    \begin{align*}
        \textstyle \sum_{m=1}^M w_m \Pi_{t=1}^d \mu_m^*(t) \ge (1/2)^d + \epsilon.
    \end{align*}
\end{lemma}
Intuitively, the moment-matching example (up to degree $d - 1$) would require to explore almost all possible length $d = H$ sequence of actions, since there would no information gain if a wrong sequence of actions $a_{1:H} \neq a_{1:H}^*$ is played. We show that any $\epsilon$-optimal policy for any $\epsilon \le (2d)^{-2d}$ needs to play the correct sequence with non-negligible probability:
\begin{lemma}
    \label{lemma:optimal_policy_lower_bound}
    Let $\mM$ be the lower-bound instance constructed with Lemma \ref{lemma:moment_matching_lower_bound} with $\epsilon \le (2d)^{-2d}$ and $d > 4$. The optimal cumulative rewards for $\mM$ is at least $(d/2) + \epsilon \cdot 2^{d-2}$. Furthermore, let $\pi_\epsilon$ be an $\epsilon$-optimal policy for $\mM$ with , then we have $\PP_{\pi_{\epsilon}} (a_{1:H} = a^*_{1:H}) \ge 1/4$.
\end{lemma}
Note that $(2d)^{-2d} = M^{-O(\sqrt{M})}$ can still be significantly larger than $(SA)^{-\Omega(\sqrt{M})}$. To formalize the lower bound argument, we can use the fundamental equality on information gain with bandit feedback:
\begin{lemma}
    \label{lemma:information_equality}
    Let $\psi$ be any exploration strategy in RMMDPs for $K$ episodes. Let $\mM^{(1)}$ and $\mM^{(2)}$ be two RMMDPs with the same transition and initial state probabilities. Let $N_{\psi, x_{1:H}}(K)$ be the number of times that a trajectory $\tau$ ends up with a sequence of state-actions $x_{1:H}$ for $K$ episodes. Then,
    \begin{align}
        \sum_{x_{1:H}} \Exs^{(1)} \left[N_{\psi, x_{1:H}}(K) \right] \cdot \KL\left( \PP^{(1)} (\cdot|x_{1:H}), \PP^{(2)} (\cdot|x_{1:H}) \right) = \KL \left(\PP_{\psi}^{(1)} (\tau^{1:K}), \PP_{\psi}^{(2)} (\tau^{1:K}) \right), \label{eq:data_processing_equality}
    \end{align}
    where $\PP_{\psi} (\tau^{1:K})$ is a distribution of $K$ trajectories obtained with the exploration strategy $\psi$, and $\PP (\cdot|x_{1:H})$ is a marginal probability of a reward sequence $r_{1:H}$ obtained from a fixed test $x_{1:H}$.
\end{lemma}
Let $\mM^{(1)}$ be the base system where rewards are always uniformly distributed over $\{0,1\}$, and $\mM^{(2)} = \mM$ be the moment-matching system from Lemma \ref{lemma:moment_matching_lower_bound}. By construction, the left hand side of equation \eqref{eq:data_processing_equality} is 0 except for the correct state-action sequence $x_{1:H}^*$. On the other hand, all information from the first model is symmetric over all sequences of (state)-actions, and thus for any exploration strategy $\psi$, there must exist at least one $x_{1:H}^{\psi}$ sequence such that $\Exs^{(1)} \left[N_{\psi, x_{1:H}^{\psi}}(K) \right] \le A^{-H} \cdot K$. Thus for the moment-matching instance $\mM^{(2)}$ with $x_{1:H}^* = x_{1:H}^{\psi}$, we can distinguish the two systems from trajectory observations $\tau^{1:K}$ only after $K = \Omega(A^H)$ episodes by Le Cam's two-point method \cite{lecam1973convergence}. We can translate this argument into a $\Omega(A^H)$ lower bound for learning general RMMDPs, and using the action-amplification argument used in \cite{kwon2021rl}, we can obtain an $(SA)^{\Omega(\sqrt{M})}$ lower bound with $d = H = \Omega(\sqrt{M})$. 
\begin{theorem}[Lower Bound for RMMDPs]
    \label{theorem:lower_bound}
    There exists a universal constant $C > 0$ and a class of RMMDPs such that to obtain an $\epsilon$-optimal policy for $\epsilon < M^{-C  \sqrt{M}}$, we need at least $(SA)^{\Omega(\sqrt{M})}/\epsilon^2$ episodes.
\end{theorem}
The proof of Theorem \ref{theorem:lower_bound} follows from Lemma \ref{lemma:moment_matching_lower_bound}-\ref{lemma:information_equality}, and the full proof can be found in Appendix \ref{appendix:proof_lower_bound}.

\begin{remark}[Fundamental gaps between full-information and bandit feedback] Consider a full feedback reward setting, {\it i.e.,} when reward samples of all state-actions are observed every time-step in each episode. In such case, a polynomial sample complexity $\poly(M)$ is possible with a folklore tournament argument \cite{devroye2001combinatorial} (by learning distributions of the entire system in time exponential in the number of samples). In contrast, we consider the bandit feedback setting, {\it i.e.,} the reward is observed only for state-action pairs that were observed within an episode, which results in sub-exponential lower bound $\exp(\Omega(\sqrt{M}))$. We believe this is an interesting sample-complexity gap between bandit and full-information feedback settings. \yon{what is the explicit gap between these results?} 
\end{remark}

\section{Conclusion}
In this work, we resolve several major open questions raised in \cite{kwon2021reinforcement}. We design the \algname algorithm and establish an $O(SA)^{O(M)}/\epsilon^2$ upper bound for learning an $\epsilon$-optimal policy of a general RMMDP. Hence, a near optimal policy of an RMMDP can be efficiently learned for $M = O(1)$. We compliment our upper bound with $(SA)^{\Omega(\sqrt{M})} / \epsilon^2$ lower bound. The result for the special case $M=2$ is further improved from \cite{kwon2021reinforcement}. Future questions can include investing the gap between the upper and lower bounds, as well as suggesting  natural assumptions that can assist in reducing the sample complexity further. Finally, designing a practical algorithm that can operate in large-scale RMMDP problems is an interesting next step to take.

\bibliographystyle{abbrv}
\bibliography{main}

\begin{thebibliography}{10}

\bibitem{azar2017minimax}
M.~G. Azar, I.~Osband, and R.~Munos.
\newblock Minimax regret bounds for reinforcement learning.
\newblock In {\em International Conference on Machine Learning}, pages
  263--272. PMLR, 2017.

\bibitem{azizzadenesheli2016reinforcement}
K.~Azizzadenesheli, A.~Lazaric, and A.~Anandkumar.
\newblock Reinforcement learning of pomdps using spectral methods.
\newblock In {\em Conference on Learning Theory}, pages 193--256, 2016.

\bibitem{bellemare2016unifying}
M.~G. Bellemare, S.~Srinivasan, G.~Ostrovski, T.~Schaul, D.~Saxton, and
  R.~Munos.
\newblock Unifying count-based exploration and intrinsic motivation.
\newblock In {\em Proceedings of the 30th International Conference on Neural
  Information Processing Systems}, pages 1479--1487, 2016.

\bibitem{boots2011closing}
B.~Boots, S.~M. Siddiqi, and G.~J. Gordon.
\newblock Closing the learning-planning loop with predictive state
  representations.
\newblock {\em The International Journal of Robotics Research}, 30(7):954--966,
  2011.

\bibitem{brunskill2013sample}
E.~Brunskill and L.~Li.
\newblock Sample complexity of multi-task reinforcement learning.
\newblock In {\em Uncertainty in Artificial Intelligence}, page 122. Citeseer,
  2013.

\bibitem{cai2022reinforcement}
Q.~Cai, Z.~Yang, and Z.~Wang.
\newblock Reinforcement learning from partial observation: Linear function
  approximation with provable sample efficiency.
\newblock In {\em International Conference on Machine Learning}, pages
  2485--2522. PMLR, 2022.

\bibitem{cesa2006prediction}
N.~Cesa-Bianchi and G.~Lugosi.
\newblock {\em Prediction, learning, and games}.
\newblock Cambridge university press, 2006.

\bibitem{chades2012momdps}
I.~Chad{\`e}s, J.~Carwardine, T.~Martin, S.~Nicol, R.~Sabbadin, and O.~Buffet.
\newblock {MOMDPs}: a solution for modelling adaptive management problems.
\newblock In {\em Twenty-Sixth AAAI Conference on Artificial Intelligence
  (AAAI-12)}, 2012.

\bibitem{chan2013learning}
S.-O. Chan, I.~Diakonikolas, X.~Sun, and R.~A. Servedio.
\newblock Learning mixtures of structured distributions over discrete domains.
\newblock In {\em Proceedings of the twenty-fourth annual ACM-SIAM symposium on
  Discrete algorithms}, pages 1380--1394. SIAM, 2013.

\bibitem{chen2019beyond}
S.~Chen and A.~Moitra.
\newblock Beyond the low-degree algorithm: mixtures of subcubes and their
  applications.
\newblock In {\em Proceedings of the 51st Annual ACM SIGACT Symposium on Theory
  of Computing}, pages 869--880, 2019.

\bibitem{devroye2001combinatorial}
L.~Devroye and G.~Lugosi.
\newblock {\em Combinatorial methods in density estimation}.
\newblock Springer Science \& Business Media, 2001.

\bibitem{doss2020optimal}
N.~Doss, Y.~Wu, P.~Yang, and H.~H. Zhou.
\newblock Optimal estimation of high-dimensional gaussian mixtures.
\newblock {\em arXiv preprint arXiv:2002.05818}, 2020.

\bibitem{duff2002optimal}
M.~O. Duff.
\newblock {\em Optimal Learning: Computational procedures for Bayes-adaptive
  Markov decision processes}.
\newblock University of Massachusetts Amherst, 2002.

\bibitem{efroni2022provable}
Y.~Efroni, C.~Jin, A.~Krishnamurthy, and S.~Miryoosefi.
\newblock Provable reinforcement learning with a short-term memory.
\newblock {\em arXiv preprint arXiv:2202.03983}, 2022.

\bibitem{feldman2008learning}
J.~Feldman, R.~O'Donnell, and R.~A. Servedio.
\newblock Learning mixtures of product distributions over discrete domains.
\newblock {\em SIAM Journal on Computing}, 37(5):1536--1564, 2008.

\bibitem{freund1999estimating}
Y.~Freund and Y.~Mansour.
\newblock Estimating a mixture of two product distributions.
\newblock In {\em Proceedings of the twelfth annual conference on Computational
  learning theory}, pages 53--62, 1999.

\bibitem{garivier2019explore}
A.~Garivier, P.~M{\'e}nard, and G.~Stoltz.
\newblock Explore first, exploit next: The true shape of regret in bandit
  problems.
\newblock {\em Mathematics of Operations Research}, 44(2):377--399, 2019.

\bibitem{gentile2014online}
C.~Gentile, S.~Li, and G.~Zappella.
\newblock Online clustering of bandits.
\newblock In {\em International Conference on Machine Learning}, pages
  757--765, 2014.

\bibitem{golowich2022learning}
N.~Golowich, A.~Moitra, and D.~Rohatgi.
\newblock Learning in observable pomdps, without computationally intractable
  oracles.
\newblock {\em arXiv preprint arXiv:2206.03446}, 2022.

\bibitem{gordon2021source}
S.~Gordon, B.~H. Mazaheri, Y.~Rabani, and L.~Schulman.
\newblock Source identification for mixtures of product distributions.
\newblock In {\em Conference on Learning Theory}, pages 2193--2216. PMLR, 2021.

\bibitem{guo2016pac}
Z.~D. Guo, S.~Doroudi, and E.~Brunskill.
\newblock A pac rl algorithm for episodic pomdps.
\newblock In {\em Artificial Intelligence and Statistics}, pages 510--518,
  2016.

\bibitem{hallak2015contextual}
A.~Hallak, D.~Di~Castro, and S.~Mannor.
\newblock Contextual markov decision processes.
\newblock {\em arXiv preprint arXiv:1502.02259}, 2015.

\bibitem{hu2021near}
J.~Hu, X.~Chen, C.~Jin, L.~Li, and L.~Wang.
\newblock Near-optimal representation learning for linear bandits and linear
  {RL}.
\newblock In {\em International Conference on Machine Learning}, pages
  4349--4358. PMLR, 2021.

\bibitem{jain2014learning}
P.~Jain and S.~Oh.
\newblock Learning mixtures of discrete product distributions using spectral
  decompositions.
\newblock In {\em Conference on Learning Theory}, pages 824--856. PMLR, 2014.

\bibitem{jaksch2010near}
T.~Jaksch, R.~Ortner, and P.~Auer.
\newblock Near-optimal regret bounds for reinforcement learning.
\newblock {\em Journal of Machine Learning Research}, 11:1563--1600, 2010.

\bibitem{kaufmann2021adaptive}
E.~Kaufmann, P.~M{\'e}nard, O.~D. Domingues, A.~Jonsson, E.~Leurent, and
  M.~Valko.
\newblock Adaptive reward-free exploration.
\newblock In {\em Algorithmic Learning Theory}, pages 865--891. PMLR, 2021.

\bibitem{kober2013reinforcement}
J.~Kober, J.~A. Bagnell, and J.~Peters.
\newblock Reinforcement learning in robotics: A survey.
\newblock {\em The International Journal of Robotics Research},
  32(11):1238--1274, 2013.

\bibitem{krishnamurthy2016pac}
A.~Krishnamurthy, A.~Agarwal, and J.~Langford.
\newblock {PAC} reinforcement learning with rich observations.
\newblock In {\em Advances in Neural Information Processing Systems}, pages
  1840--1848, 2016.

\bibitem{kwon2021reinforcement}
J.~Kwon, Y.~Efroni, C.~Caramanis, and S.~Mannor.
\newblock Reinforcement learning in reward-mixing mdps.
\newblock {\em Advances in Neural Information Processing Systems}, 34, 2021.

\bibitem{kwon2021rl}
J.~Kwon, Y.~Efroni, C.~Caramanis, and S.~Mannor.
\newblock {RL} for latent mdps: Regret guarantees and a lower bound.
\newblock {\em Advances in Neural Information Processing Systems}, 34, 2021.

\bibitem{kwon2022coordinated}
J.~Kwon, Y.~Efroni, C.~Caramanis, and S.~Mannor.
\newblock Coordinated attacks against contextual bandits: Fundamental limits
  and defense mechanisms.
\newblock In {\em Proceedings of the 39th International Conference on Machine
  Learning}, pages 11772--11789. PMLR, 2022.

\bibitem{kwon2021minimax}
J.~Kwon, N.~Ho, and C.~Caramanis.
\newblock On the minimax optimality of the em algorithm for learning
  two-component mixed linear regression.
\newblock In {\em International Conference on Artificial Intelligence and
  Statistics}, pages 1405--1413. PMLR, 2021.

\bibitem{lecam1973convergence}
L.~LeCam.
\newblock Convergence of estimates under dimensionality restrictions.
\newblock {\em The Annals of Statistics}, pages 38--53, 1973.

\bibitem{liu2022partially}
Q.~Liu, A.~Chung, C.~Szepesv{\'a}ri, and C.~Jin.
\newblock When is partially observable reinforcement learning not scary?
\newblock {\em arXiv preprint arXiv:2204.08967}, 2022.

\bibitem{liu2016pac}
Y.~Liu, Z.~Guo, and E.~Brunskill.
\newblock {PAC} continuous state online multitask reinforcement learning with
  identification.
\newblock In {\em Proceedings of the 2016 International Conference on
  Autonomous Agents \& Multiagent Systems}, pages 438--446, 2016.

\bibitem{maillard2014latent}
O.-A. Maillard and S.~Mannor.
\newblock Latent bandits.
\newblock In {\em International Conference on Machine Learning}, pages
  136--144, 2014.

\bibitem{mnih2013playing}
V.~Mnih, K.~Kavukcuoglu, D.~Silver, A.~Graves, I.~Antonoglou, D.~Wierstra, and
  M.~Riedmiller.
\newblock Playing atari with deep reinforcement learning.
\newblock {\em arXiv preprint arXiv:1312.5602}, 2013.

\bibitem{moitra2010settling}
A.~Moitra and G.~Valiant.
\newblock Settling the polynomial learnability of mixtures of gaussians.
\newblock In {\em 2010 IEEE 51st Annual Symposium on Foundations of Computer
  Science}, pages 93--102. IEEE, 2010.

\bibitem{silver2018general}
D.~Silver, T.~Hubert, J.~Schrittwieser, I.~Antonoglou, M.~Lai, A.~Guez,
  M.~Lanctot, L.~Sifre, D.~Kumaran, T.~Graepel, et~al.
\newblock A general reinforcement learning algorithm that masters chess, shogi,
  and go through self-play.
\newblock {\em Science}, 362(6419):1140--1144, 2018.

\bibitem{smallwood1973optimal}
R.~D. Smallwood and E.~J. Sondik.
\newblock The optimal control of partially observable markov processes over a
  finite horizon.
\newblock {\em Operations research}, 21(5):1071--1088, 1973.

\bibitem{steimle2018multi}
L.~N. Steimle, D.~L. Kaufman, and B.~T. Denton.
\newblock Multi-model markov decision processes.
\newblock {\em Optimization Online URL http://www. optimization-online.
  org/DB\_FILE/2018/01/6434. pdf}, 2018.

\bibitem{tang2017exploration}
H.~Tang, R.~Houthooft, D.~Foote, A.~Stooke, X.~Chen, Y.~Duan, J.~Schulman,
  F.~De~Turck, and P.~Abbeel.
\newblock \# exploration: A study of count-based exploration for deep
  reinforcement learning.
\newblock In {\em 31st Conference on Neural Information Processing Systems
  (NIPS)}, volume~30, pages 1--18, 2017.

\bibitem{taylor2009transfer}
M.~E. Taylor and P.~Stone.
\newblock Transfer learning for reinforcement learning domains: A survey.
\newblock {\em Journal of Machine Learning Research}, 10(7), 2009.

\bibitem{uehara2022provably}
M.~Uehara, A.~Sekhari, J.~D. Lee, N.~Kallus, and W.~Sun.
\newblock Provably efficient reinforcement learning in partially observable
  dynamical systems.
\newblock {\em arXiv preprint arXiv:2206.12020}, 2022.

\bibitem{wainwright2019high}
M.~J. Wainwright.
\newblock {\em High-dimensional statistics: A non-asymptotic viewpoint},
  volume~48.
\newblock Cambridge University Press, 2019.

\bibitem{zhan2022pac}
W.~Zhan, M.~Uehara, W.~Sun, and J.~D. Lee.
\newblock Pac reinforcement learning for predictive state representations.
\newblock {\em arXiv preprint arXiv:2207.05738}, 2022.

\bibitem{zintgraf2019varibad}
L.~Zintgraf, K.~Shiarlis, M.~Igl, S.~Schulze, Y.~Gal, K.~Hofmann, and
  S.~Whiteson.
\newblock {VariBAD}: A very good method for bayes-adaptive deep {RL} via
  meta-learning.
\newblock In {\em International Conference on Learning Representations}, 2019.

\end{thebibliography}

\newpage 

\appendix

\begin{appendices}
\section{Proof of Main Theorem \ref{theorem:sample_complexity_upper_bound}}
\label{appendix:proof_main_lemma}

\subsection{Proof of Lemma \ref{lemma:eventwise_total_bound}}
Note that invoking Theorem \ref{theorem:distance_from_moment_theorem}, it directly follows that $\Eps_l' \subseteq \Eps_{tot}$ with setting $\delta := \delta_l$ for $\mX_d$. Lemma \ref{lemma:eventwise_total_bound} follows the conclusion of Theorem \ref{theorem:distance_from_moment_theorem} with $M_1 = M_2 = M$ and $d = \min(2M-1, H)$.

\subsection{Proof of Theorem \ref{theorem:distance_from_moment_theorem}}
We start by unfolding the expression of statistical distance,
\begin{align*}
    \sum_{\tau: x_{1:H} \in \Eps} | \PP^{(1)}_\pi (\tau) - \PP^{(2)}_\pi(\tau) | &= \sum_{x_{1:H}} \sum_{r_{1:H}} \indic{x_{1:H} \in \Eps} \nu(s_1) \Pi_{t=1}^{H-1} T(s_{t+1}|s_{t}, a_{t}) \cdot \Pi_{t=1}^H \pi(a_t | h_t) \\
    &\qquad \times \left| \sum_{m=1}^{M_1} w_m^{(1)} \Pi_{t=1}^H \mu_m^{(1)} (x_t, r_t) - \sum_{m=1}^{M_2} w_m^{(2)} \Pi_{t=1}^H \mu_m^{(2)} (x_t, r_t) \right|.
\end{align*}
With slight abuse in notation, we compactly define $T_{1:t} := v(s_1) \cdot \Pi_{t'=1}^{t-1} T(s_{t'+1} | s_{t'}, a_{t'})$ and $\pi_{1:t} := \Pi_{t'=1}^t \pi(a_{t'} | h_{t'})$. Let $h_t := ((s,a,r)_{1:{t-1}}, s_t)$ be a history before taking an action at $t^{th}$ time step. The above can be rewritten as
\begin{align*}
    \sum_{\tau: x_{1:H} \in \Eps} | \PP^{(1)}_\pi (\tau) - \PP^{(2)}_\pi(\tau) | &= \sum_{x_{1:H}} \sum_{r_{1:H}} \indic{\Eps} T_{1:H} \cdot \pi_{1:H} \cdot \left| \moment^{(1)} (x_{1:H}, r_{1:H}) - \moment^{(2)} (x_{1:H}, r_{1:H}) \right|.
\end{align*}

We adopt the brilliant idea of mathematical induction on the number of latent contexts, which is first employed in \cite{chen2019beyond} for a related problem of learning mixtures of product distributions. Specifically, the authors in \cite{chen2019beyond} have shown the statistical closeness between mixtures of product distributions from matching higher-order multilinear moments. The key to proceed with the mathematical induction is to reduce the number of contexts {\it or} the length of sequence at least by one whenever we process one time-step event. 

To apply the mathematical induction, we first need to check the base case. The base case for Theorem \ref{theorem:distance_from_moment_theorem} is when $M_1 = M_2 = 1$ or $H = 1$. Before we proceed, we define a few definitions on the probability of encountering trajectories of interest.

\paragraph{Additional Notation} Let us denote the probability of ending up with a trajectory $\tau$ conditioned on a history $h$, such that the $x_{1:H}$ part belongs to $\Eps$ as:
\begin{align}
    \label{eq:def_eps_*}
    \PP^* (\Eps| h) := \sup_{\pi \in \Pi} \PP_{\pi} (\Eps | h). 
\end{align}
By definition, we have the following inequalities, for history $h_t = ((s,a,r)_{1:t-1}, s_t)$ at time $t$, action $a_t$ and any history-dependent policy $\pi$, we have
\begin{align}
    \PP^* (\Eps| h_{H}) \ge \sum_{a_H} \indic{x_{1:H} \in \Eps} \pi(a_H| h_H), &\ \quad t = H \label{lemma:bellman_for_max_event_probability3}, \\
    \PP^* (\Eps| h_{t}) \ge \sum_{a_t} \PP^*(\Eps| h_{t}, a_t) \pi(a_t| h_t), &\ \quad t < H \label{lemma:bellman_for_max_event_probability1}, \\
    \PP^* (\Eps| h_{t}, a_t) \ge \sum_{s_{t+1}} \PP^*(\Eps| h_{t+1}) T(s_{t+1} | s_t, a_t), &\ \quad t < H \label{lemma:bellman_for_max_event_probability2}.
\end{align}
Also, since $\PP_{\Eps}^*(\cdot)$ only depends on the occurance of $x_{1:H}$, any two RMMDP models with the same transition and initial distribution have the same value for $\PP_{\Eps}^*$:
\begin{align*}
    \PP_\Eps^* (h) = \sup_{\pi \in \Pi} \PP^{(1)}_{\pi} (\Eps|h) = \sup_{\pi \in \Pi} \PP^{(2)}_{\pi} (\Eps|h).
\end{align*}
Hence when we consider the same transition model, we often omit $(1)$ and $(2)$ in superscript from $\PP^{(1)}_{\pi}(\Eps| h)$ or $\PP^{(2)}_{\pi} (\Eps| h)$. Also note that $\PP_{\Eps}^*(h_t) = \PP_{\Eps}^*(x_{1:t-1}, s_t)$ and $\PP_{\Eps}^*(h_t, a_t) = \PP_{\Eps}^*(x_{1:t})$ since $\PP_{\Eps}^*$ only depends on the state-action part of the history.

\paragraph{Base Case I:} When $M_1 = M_2 = 1$, note that
\begin{align*}
    \moment(x_{1:H}, r_{1:H}) = \Pi_{t=1}^H \mu(x_t, r_t),
\end{align*}
where we can omit the subscript $m$ for reward models $\mu$ since there is only one latent context. Then, 
\begin{align*}
    \sum_{\tau: x_{1:H} \in \Eps} &| \PP^{(1)}_{\pi} (\tau) - \PP^{(2)}_{\pi} (\tau) | = \sum_{x_{1:H}} \sum_{r_{1:H}} \indic{\Eps} \pi_{1:H} T_{1:H} \left| \Pi_{t=1}^{H} \mu^{(1)} (x_t, r_t) - \Pi_{t=1}^{H} \mu^{(2)}(x_t, r_t) \right| \\
    &\le \sum_{x_{1:H}} \sum_{r_{1:H-1}} \indic{\Eps} \pi_{1:H} T_{1:H} \left| \Pi_{t=1}^{H-1} \mu^{(1)} (x_t, r_t) - \Pi_{t=1}^{H-1} \mu^{(2)}(x_t, r_t) \right| \cdot \left( \sum_{r_{H}} \mu^{(1)}(x_H, r_H) \right) \\
    &\ \ + \sum_{x_{1:H}} \sum_{r_{1:H-1}} \indic{\Eps} \pi_{1:H} T_{1:H} \left(\Pi_{t=1}^{H-1} \mu^{(2)} (x_t, r_t)\right) \cdot \sum_{r_{H}} \left|  \mu^{(1)}(x_H, r_H) - \mu^{(2)} (x_H, r_H) \right|.
\end{align*}
By the moment-closeness condition, note that we have $\left|  \mu^{(1)}(x_{H}, r_{H}) - \mu^{(2)} (x_{H}, r_{H}) \right| \le \delta$. We also know that for any $x_{H} \in \mS \times \mA$, we have $\sum_{r_{H}} \mu^{(1)} (x_{H}, r_{H}) = 1$. On one hand, it is easy to verify that
\begin{align*}
    \sum_{x_{1:H}} &\sum_{r_{1:H-1}} \indic{\Eps} \pi_{1:H} T_{1:H} \Pi_{t=1}^{H-1} \mu^{(2)} (x_t, r_t) \\
    &= \sum_{x_{1:H-1}} \sum_{r_{1:H-1}} \pi_{1:H-1} T_{1:H-1} \Pi_{t=1}^{H-1} \mu^{(2)} (x_t, r_t) \sum_{s_H} T(s_H | x_{H-1}) \sum_{a_H} \indic{\Eps} \pi(a_H | h_H) \\
    &\le \sum_{x_{1:H-1}} \sum_{r_{1:H-1}} \pi_{1:H-1} T_{1:H-1} \Pi_{t=1}^{H-1} \mu^{(2)} (x_t, r_t) \sum_{s_H} T(s_H | x_{H-1}) \PP_{\Eps}^* (h_H) \\
    &\le \sum_{x_{1:H-1}} \sum_{r_{1:H-2}} \pi_{1:H-1} T_{1:H-1} \left( \Pi_{t=1}^{H-2} \mu^{(2)} (x_t, r_t) \right) \PP_{\Eps}^* (h_{H-1}, a_{H-1}) \sum_{r_{H-1}} \mu^{(2)} (x_{H-1}, r_{H-1}) \\
    &= \sum_{x_{1:H-1}} \sum_{r_{1:H-2}} \pi_{1:H-1} T_{1:H-1} \left(\Pi_{t=1}^{H-2} \mu^{(2)} (x_t, r_t)\right) \PP_{\Eps}^* (h_{H-1}, a_{H-1}) \le ... \le \PP_{\Eps}^* (\phi),
\end{align*}
where the inequalities come from \eqref{lemma:bellman_for_max_event_probability2} and \eqref{lemma:bellman_for_max_event_probability3}. We also have
\begin{align*}
    \sum_{x_{1:H}} &\sum_{r_{1:H-1}} \indic{\Eps} \pi_{1:H} T_{1:H} \left| \Pi_{t=1}^{H-1} \mu^{(1)} (x_t, r_t) - \Pi_{t=1}^{H-1} \mu^{(2)}(x_t, r_t) \right| \\
    &\le \sum_{x_{1:H-1}} \sum_{r_{1:H-1}} \PP_{\Eps}^* (h_{H-1}, a_{H-1}) \pi_{1:H-1} T_{1:H-1} \left| \Pi_{t=1}^{H-1} \mu^{(1)} (x_t, r_t) - \Pi_{t=1}^{H-1} \mu^{(2)}(x_t, r_t) \right|.
\end{align*}
Applying the same argument recursively, we can proceed from $t=H$ to $t=1$ and get:
\begin{align*}
    \sum_{x_{1:H}} &\sum_{r_{1:H}} \indic{\Eps} \pi_{1:H-1} T_{1:H} \left| \Pi_{t=1}^{H} \mu^{(1)} (x_t, r_t) - \Pi_{t=1}^{H} \mu^{(2)}(x_t, r_t) \right| \le \PP_{\Eps}^*(\phi) HZ\delta. 
\end{align*}

\paragraph{Base Case II:} If $H \le M_1 + M_2 - 1$, then by the moment-closeness condition, 
\begin{align*}
    \sum_{\tau: x_{1:H} \in \Eps} &| \PP^{1,\pi} (\tau) - \PP^{2,\pi} (\tau) | = \sum_{x_{1:H}} \sum_{r_{1:H}} \indic{\Eps} \pi_{1:H} T_{1:H} \left| \moment^{(1)} (x_{1:H}, r_{1:H}) - \moment^{(2)} (x_{1:H}, r_{1:H}) \right| \\
    &\le \delta \cdot \sum_{x_{1:H}} \sum_{r_{1:H}} \indic{\Eps} \pi_{1:H} T_{1:H} \\
    &= \delta \cdot \sum_{x_{1:H-1}} \sum_{r_{1:H-1}} \pi_{1:H-1} T_{1:H-1} \sum_{x_H} \indic{\Eps} \pi(a_H | h_H) T(s_H | s_{H-1}, a_{H-1}) \sum_{r_{H}} 1 \\
    &\le Z\delta \cdot \sum_{x_{1:H-1}} \sum_{r_{1:H-1}} \pi_{1:H-1} T_{1:H-1} \PP_{\Eps}^* (h_{H-1}, a_{H-1}) \le \ldots \le \PP_{\Eps}^*(\phi) \cdot Z^H \delta.
\end{align*}
where inequalities hold due to the moment matching condition and inequalities for $\PP_{\Eps}^*(\cdot)$.

\paragraph{Induction on $H$ and $M_1 + M_2$.} Suppose that an inductive assumption is true for all two RMMDP models when the total number of latent contexts is less than $M_1 + M_2$, or when the length of episode is less than $H$. Let $l(x,z)$ be the smallest probability among all latent contexts, {\it i.e.,}
\begin{align*}
    l(x,z) := \min \left( \min_{m \in [M_1]} \mu_m^{(1)} (x,z),  \min_{m \in [M_2]} \mu_m^{(2)}(x,z) \right).
\end{align*}
Note that the moment-closeness condition says that for any $1 \le t_1 < t_2 < \ldots < t_d \le H$, 
\begin{align*}
    \left| \sum_{m=1}^{M_1} w^{(1)}_m \Pi_{q=1}^d \mu^{(1)}_m (x_{t_q}, z_{q}) - \sum_{m=1}^{M_2} w^{(2)}_m \Pi_{q=1}^d \mu^{(2)}_m (x_{t_q}, z_{q}) \right| \le \delta, \qquad \forall x_{1:H} \in \Eps, \{z_q\}_{q=1}^d \in \mZ^{\bigotimes d},
\end{align*}
and similarly for degree $d-1$ moments of any parts of trajectories in $\Eps$. Let us fix the event happened at $t=1$ as $(x_1, r_1)$. Without loss of generality, suppose that the minimum for $l(x_1, r_1)$ is achieved from the first RMMDP model $\mM^{(1)}$. Define $p^{(1)} := \sum_{m=1}^{M_1} w_{m}^{(1)} \left(\mu_m^{(1)} (x_1, r_1) - l(x_1, r_1) \right)$ and $p^{(2)} := \sum_{m=1}^{M_1} w_{m}^{(2)} \left(\mu_m^{(2)} (x_1, r_1) - l(x_1, r_1) \right)$. By the moment closeness condition, $\left| p^{(1)} - p^{(2)} \right| \le \delta$. Note that in each model, we can decompose the probability of each trajectory $\tau = (x,r)_{1:H}$ as
\begin{align*}
    \pi_{1:H} T_{1:H} \cdot &\sum_{m=1}^{M_1} w^{(1)}_m \Pi_{t=1}^{H} \mu_m^{(1)} (x_t, r_t) \\
    &= \PP^{(1)}_{\pi} (x_1, r_1) l(x_1, r_1) \pi_{2:H} T_{1:H-1} \cdot \sum_{m=1}^{M_1} w^{(1)}_m \Pi_{t=2}^{H} \mu_m^{(1)} (x_t, r_t) \\
    &\ + \PP^{(1)}_{\pi} (x_1, r_1) \pi_{2:H} T_{1:H-1} \cdot \sum_{m=1}^{M_1} w^{(1)}_m (\mu_m^{(1)} (x_1, r_1) - l(x_1, r_1)) \Pi_{t=2}^{H} \mu_m^{(1)} (x_t, r_t).
\end{align*}
Let us define two auxiliary models $\mM^{(3)}$ and $\mM^{(4)}$ as the following: 
\begin{enumerate}
    \item $\mM^{(3)}$ has the transition model $T^{(3)} (\cdot) := T^{(1)} (\cdot)$, initial state distribution $\nu^{(3)} (\cdot) := T^{(1)} (\cdot| s_1, r_1)$, latent reward models $\mu_m^{(3)} (\cdot) := \mu_m^{(1)} (\cdot)$, and mixing weights $w_m^{(3)} := \frac{1}{p^{(1)}} w_{m}^{(1)}(\mu_m^{(1)} (x_1, r_1) - l(x_1, r_1))$. 
    \item $\mM^{(4)}$ is defined similarly as $\mM^{(3)}$ from $\mM^{(2)}$, except for the mixing weights $w^{(4)}_m := \frac{1}{p^{(2)}} w_{m}^{(2)}(\mu_m^{(2)} (x_1, r_1) - l(x_1, r_1))$.
\end{enumerate}
Note that $\mM^{(3)}$ has at most $M_1 - 1$ non-zero mixing weights, since $l(x_1, r_1)$ must match to one of reward probabilities $\{\mu_m^{(1)} (x_1, r_1)\}_{m=1}^{M_1}$. Hence we can leave out only non-zero mixing weights in $\{w_m^{(3)}\}_{m=1}^{M_1}$ and consider as if there are only $M_1 - 1$ latent contexts in $\mM^{(3)}$. 

Let us define $\Eps_{x_1} := \{x_{2:H} | (x_1, x_{2:H}) \in \Eps \}$, a subset of trajectories in $\Eps$ starting from $x_1$. Note that $\PP^{\pi} (x_{2:H} \in \Eps_{x_1} | x_1) \le \PP_{\Eps}^*(x_1)$. We can decompose the statistical distance of trajectories as the following:
\begin{align*}
    &\sum_{x_{1:H}} \sum_{r_{1:H}} \indic{\Eps} \pi_{1:H} T_{1:H} \left| \sum_{m=1}^{M_1} w_m^{(1)} \Pi_{t=1}^{H} \mu_m^{(1)} (x_t, r_t) - \sum_{m=1}^{M_2} w^{(2)}_m \Pi_{t=2}^H \mu_m^{(2)} (x_t, r_t) \right| \\
    &\le \sum_{x_1, r_1} \PP^{1,\pi}(x_1) l(x_1, r_1) \cdot \sum_{x_{2:H}, r_{2:H}} \indic{\Eps_{x_1}} \pi_{2:H} T_{2:H} \left| \moment^{(1)} (x_{2:H}, r_{2:H}) - \moment^{(2)} (x_{2:H}, r_{2:H}) \right| \\
    &\ + \sum_{x_1, r_1} \PP^{1,\pi}(x_1) \cdot \sum_{x_{2:H}, r_{2:H}} \indic{\Eps_{x_1}} \pi_{2:H} T_{2:H} \cdot \left| p^{(1)} \moment^{(3)} (x_{2:H}, r_{2:H}) - p^{(2)} \moment^{(4)} (x_{2:H}, r_{2:H}) \right|.
\end{align*}
We observe that in the first term, the summation starting from $t = 2$ to $H$ can be considered as a statistical difference between two RMMDP models $\mM^{(1)}$ and $\mM^{(2)}$ with a new common initial state distribution $\nu'(\cdot) := T(\cdot | s_1, a_1)$ in a shorter time-horizon of length $H-1$. A new policy $\pi'(\cdot | h)$ is $\pi(\cdot|h, r_1, x_1)$ in this setup. Note that the moment-closeness condition remains the same, and therefore by inductive assumption on $H$, we have
\begin{align*}
    \sum_{x_{2:H}, r_{2:H}} &\indic{\Eps_{x_1}} \pi_{2:H} T_{2:H} \left| \moment^{(1)} (x_{2:H}, r_{2:H}) - \moment^{(2)} (x_{2:H}, r_{2:H}) \right| \le \PP_{\Eps}^* (x_1) \cdot (4(H-1)Z)^d \delta. 
\end{align*}
For the second term, we show later that
\begin{align}
    \sum_{x_{2:H}, r_{2:H}} &\indic{\Eps_{x_1}} \pi_{2:H} T_{2:H} \cdot \left| p^{(1)}\moment^{(3)} (x_{2:H}, r_{2:H}) - p^{(2)} \moment^{(4)} (x_{2:H}, r_{2:H}) \right| \nonumber \\
    &\le \PP_{\Eps}^*(x_1) \cdot 4 \cdot (4(H-1)Z)^{d-1} \delta. 
    \label{eq:inequality_for_auxiliary_models}
\end{align}
We prove \eqref{eq:inequality_for_auxiliary_models} in Section \ref{appendix:inequality_for_auxiliary_models}. Assuming this, we have
\begin{align*}
    &\sum_{x_{1:H}} \sum_{r_{1:H}} \indic{\Eps} \pi_{1:H} T_{1:H} \left| \sum_{m=1}^{M_1} w_m^{(1)} \Pi_{t=1}^{H} \mu_m^{(1)} (x_t, r_t) - \sum_{m=1}^{M_2} w^{(2)}_m \Pi_{t=2}^H \mu_m^{(2)} (x_t, r_t) \right| \\ 
    &\le \sum_{x_1} \PP^{(1)}_{\pi}(x_1) \PP_{\Eps}^*(x_1) \cdot \sum_{r_1} (4(H-1)Z)^d \delta \cdot l(x_1, r_1) \\
    &\ + \sum_{x_1} \PP^{(1)}_{\pi}(x_1) \PP_{\Eps}^*(x_1) \cdot \sum_{r_1} 4 \cdot (4(H-1)Z)^{d-1} \delta.
\end{align*}
Note that $\sum_{r_1} l(x_1, r_1) \le \sum_{r_1} \mu_m^{(1)} (x_1, r_1) \le 1$ for any fixed $m \in [M]$, and thus the above can be further bounded by
\begin{align*}
    \delta (4(H-1))^{d-1} Z^d \cdot \left(\sum_{x_1} \PP^{(1)}_{\pi} (x_1) \PP^{*}_{\Eps} (x_1) \right) \cdot (4(H-1) + 4) \le \PP_{\Eps}^* (\phi) \cdot (4HZ)^d \delta,
\end{align*}
which proves the result for $M_1, M_2, H$. Applying the same argument inductively for all increasing $M_1, M_2$ and $H$, the above also holds for $M_1 = M_2 = M$ and $d = \min(2M-1, H)$.

\subsubsection{Proof of Equation \eqref{eq:inequality_for_auxiliary_models}} 
\label{appendix:inequality_for_auxiliary_models}
We first separate a subtle issue of mismatch between $p^{(1)}$ and $p^{(2)}$:
\begin{align*}
    \sum_{x_{2:H}, r_{2:H}} &\indic{\Eps_{x_1}} \pi_{2:H} T_{2:H} \cdot \left| p^{(1)}\moment^{(3)} (x_{2:H}, r_{2:H}) - p^{(2)} \moment^{(4)} (x_{2:H}, r_{2:H}) \right| \\
    &\le p^{(1)} \cdot \sum_{x_{2:H}, r_{2:H}} \indic{\Eps_{x_1}} \pi_{2:H} T_{2:H} \left| \moment^{(3)} (x_{2:H}, r_{2:H}) - \moment^{(4)} (x_{2:H}, r_{2:H}) \right| \\
    &\ + |p^{(1)} - p^{(2)}| \cdot \sum_{x_{2:H}, r_{2:H}} \indic{\Eps_{x_1}} \pi_{2:H} T_{2:H} \moment^{(4)} (x_{2:H}, r_{2:H}). 
\end{align*}
Since $|p^{(1)} - p^{(2)}| \le \delta$, we have
\begin{align*}
    |p^{(1)} - p^{(2)}| \cdot \sum_{x_{2:H}, r_{2:H}} \indic{\Eps_{x_1}} \pi_{2:H} T_{2:H} \moment^{(4)} (x_{2:H}, r_{2:H}) \le \PP_{\Eps}^*(x_1) \delta.
\end{align*}
For the remaining term, we examine the moment-closeness condition for the auxiliary model $\mM^{(3)}$ and $\mM^{(4)}$. If $M_1 = 1$, then we must have $p^{(1)} = 0$ and thus the remaining term is 0. Hence we focus on the case that $M_1 > 1$. We can consider two cases: if $p^{(1)} \le 4\delta$, then instead of using the moment-closeness condition, we apply
\begin{align*}
    p^{(1)} \cdot &\sum_{x_{2:H}, r_{2:H}} \indic{\Eps_{x_1}} \pi_{2:H} T_{2:H} \left| \moment^{(3)} (x_{2:H}, r_{2:H}) - \moment^{(4)} (x_{2:H}, r_{2:H}) \right| \\
    &\le 4\delta \cdot \sum_{x_{2:H}, r_{2:H}} \indic{\Eps_{x_1}} \pi_{2:H} T_{2:H} \left( \moment^{(3)} (x_{2:H}, r_{2:H}) + \moment^{(4)} (x_{2:H}, r_{2:H}) \right) \\
    &\le \PP_{\Eps}^* (x_1) \cdot 8\delta,
\end{align*}
and we are done as long as $(4(H-1)Z)^{d-1} > 8$. Otherwise, let us compare the moments of degree up to $d-1$ in $\mM^{(3)}$ and $\mM^{(4)}$. Consider any moments consisting of $q \le d-1$ pairs of state-actions in any trajectory $x_{2:H} \in \Eps_{x_1}$ at non-overlapping time-steps $2 \le t_1 < \ldots < t_{q} \le H$. For any $\bm{z} = z_{1:q} \in \mZ^{\bigotimes q}$ with $\bm{x} = (x_{t_{q'}})_{q'=1}^q$, we can check that
\begin{align*}
    \Big| \moment^{(3)} (\bm{x}, \bm{z}) - \moment^{(4)} (\bm{x}, \bm{z}) \Big| = \left | \tssum_{m=1}^{M_1-1} w_m^{(3)} \Pi_{q'=1}^q \mu_m^{(3)} (x_{t_{q'}}, z_{q'}) - \tssum_{m=1}^{M_2} w_m^{(4)} \Pi_{q'=1}^q \mu_m^{(4)} (x_{t_{q'}}, z_{q'}) \right|.
\end{align*}
Recall that
\begin{align*}
    \sum_{m=1}^{M_1-1} w_m^{(3)} \Pi_{q'=1}^q \mu_m^{(3)} (x_{t_{q'}}, z_{q'}) = \frac{1}{p^{(1)}} \sum_{m=1}^{M_1} w_m^{(1)} (\mu_m^{(1)}(x_1, r_1) - l(x_1, r_1)) \Pi_{q'=1}^q \mu_m^{(1)} (x_{t_{q'}}, z_{q'}),
\end{align*}
and similarly for the moments in $\mM^{(4)}$. Hence we can decompose the moment difference as the following:
\begin{align*}
    \Big| &\moment^{(3)} (\bm{x}, \bm{z}) - \moment^{(4)} (\bm{x}, \bm{z}) \Big| \\
    &\le \frac{1}{p^{(1)}} \left | \sum_{m=1}^{M_1} w_m^{(1)} \mu_m^{(1)}(x_1, r_1) \Pi_{q'=1}^q \mu_m^{(1)} (x_{t_{q'}}, z_{q'}) - \sum_{m=1}^{M_2} w_m^{(2)} \mu_m^{(2)}(x_1, r_1) \Pi_{q'=1}^{q} \mu_m^{(2)} (x_{t_{q'}}, z_{q'}) \right| \\
    &+ \frac{l(x_1, r_1)}{p^{(1)}} \left | \sum_{m=1}^{M_1} w_m^{(1)} \Pi_{q'=1}^q \mu_m^{(1)} (x_{t_{q'}}, z_{q'}) - \sum_{m=1}^{M_2} w_m^{(2)} \Pi_{q'=1}^{q} \mu_m^{(2)} (x_{t_{q'}}, z_{q'}) \right| \\
    &+ \left| \frac{1}{p^{(1)}} - \frac{1}{p^{(2)}} \right| \cdot \left| \sum_{m=1}^{M_1} w_m^{(2)} (\mu_m^{(2)}(x_1, r_1) - l(x_1, r_1)) \Pi_{q'=1}^q \mu_m^{(2)} (x_{t_{q'}}, z_{q'}) \right|. 
\end{align*}
By the moment-closeness condition for $\mM^{(1)}$ and $\mM^{(2)}$ up to degree $d$, the first two terms can be easily bounded by $\frac{2\delta}{p^{(1)}}$. For the last term, note that 
\begin{align*}
    \sum_{m=1}^{M_1} w_m^{(2)} (\mu_m^{(2)}(x_1, r_1) - l(x_1, r_1)) \Pi_{q'=1}^q \mu_m^{(2)} (x_{t_{q'}}, z_{q'}) &\le \sum_{m=1}^{M_1} w_m^{(2)} (\mu_m^{(2)}(x_1, r_1) - l(x_1, r_1)) \le p^{(2)},
\end{align*}
and also $|p^{(1)} - p^{(2)}| \le \delta$, and thus the last term is bounded by $\delta / p^{(1)}$. Therefore, we can conclude that $\mM^{(3)}$ and $\mM^{(4)}$ satisfies the moment-closeness condition (regarding trajectories in $\Eps_{x_1}$) with $\delta' = 3\delta / p^{(1)}$. Applying the inductive assumption for $M_1-1, M_2$ and $H-1$, we have
\begin{align*}
    p^{(1)} \cdot &\sum_{x_{2:H}, r_{2:H}} \indic{\Eps_{x_1}} \pi_{2:H} T_{1:H-1} \left| \moment^{(3)} (x_{2:H}, r_{2:H}) - \moment^{(4)} (x_{2:H}, r_{2:H}) \right| \\
    &\le \PP_{\Eps}^* (x_1) \cdot 3\cdot (4(H-1)Z)^{d-1} \cdot \delta.
\end{align*}
Finally, we can apply the results to
\begin{align*}
    \sum_{x_{2:H}, r_{2:H}} &\indic{\Eps_{x_1}} \pi_{2:H} T_{1:H-1} \cdot p^{(1)} \left| p^{(1)}\moment^{(3)} (x_{2:H}, r_{2:H}) - p^{(2)} \moment^{(4)} (x_{2:H}, r_{2:H}) \right| \\
    &\le \PP_{\Eps}^*(x_1) (\delta + 3 \cdot (4(H-1)Z)^{d-1} \cdot \delta) \le \PP_{\Eps}^*(x_1) \cdot 4\cdot (4(H-1)Z)^{d-1} \cdot \delta.
\end{align*}
This proves \eqref{eq:inequality_for_auxiliary_models}, and thus completes Lemma \ref{lemma:eventwise_total_bound}.

\subsection{Proof of Theorem \ref{theorem:sample_complexity_upper_bound}}
\label{appendix:proof:sample_complexity_upper_bound}

We first note that
\begin{align*}
    (\mS \times \mA \times \mZ)^{\bigotimes H} = \bigcup_{l=0}^{L+1} \Eps_l',
\end{align*}
Thus, we can split the sum over all trajectories into $L+2$ levels, to bound the statistical distance between trajectory distributions (reiterating equation \eqref{eq:bound_total_variation_distance_overall}):
\begin{align*}
    \|(\PP^{(1)}_\pi - \PP^{(2)}_\pi) (\tau)\|_1 = \sum_{l=0}^{L+1} \sum_{\tau:x_{1:H} \in \Eps_l'} |\PP^{(1)}_{\pi}(\tau) - \PP^{(2)}_{\pi}(\tau) | \le \sum_{l = 0}^{L+1} \sup_{\pi \in \Pi} \PP^{(1)}_\pi( x_{1:H} \in \Eps_l' ) \cdot O(\epsilon_l).
\end{align*}
Note that this holds for all history-dependent policies. Then we apply the results from Lemma \ref{lemma:eventwise_total_bound} and \ref{lemma:reward_free_guarantee}, which yields
\begin{align*}
    \|(\PP^{(1)}_\pi - \PP^{(2)}_\pi) (\tau)\|_1 &\le \sum_{l = 0}^{L+1} \sup_{\pi \in \Pi} \PP^{(1)}_\pi( x_{1:H} \in \Eps_l' ) \cdot O(\epsilon_l) \\
    &\le \sum_{l = 0}^{L+1} H^d \epsilon_{\rm pe} \sqrt{n_l/\iota_c} \cdot (4HZ)^{d} \cdot O\left(\sqrt{\iota_c/n_l}\right) \\
    &\le O(L) H^d \epsilon_{\rm pe} (4HZ)^d.
\end{align*}
Plugging $\epsilon_{\rm pe} = \epsilon / (HL(4H^2Z)^d)$, we get $\|(\PP^{(1)}_\pi - \PP^{(2)}_\pi) (\tau)\|_1 \le O(\epsilon/H)$, which in turn gives 
\begin{align*}
    |V_{\mM^{(1)}}^{\pi} - V_{\mM^{(2)}}^{\pi}| \le H \cdot \|(\PP^{(1)}_{\pi} - \PP^{(2)}_{\pi}) ((x,r)_{1:H})\|_1 \le O(\epsilon),
\end{align*}
as desired.

\section{Proofs for Additional Results in Section \ref{section:analysis}} 
\subsection{Proof of Lemma \ref{lemma:third_moment_matching_m2}}
This equation directly comes from unfolding the expression:
\begin{align*}
    -2 &\cdot \frac{1}{8} (\mu_1(x_1,z_1) + \mu_2(x_1,z_1)) (\mu_1(x_2,z_2) + \mu_2(x_2,z_2)) (\mu_1(x_3,z_3) + \mu_2(x_3,z_3)) \\
    &+ \frac{1}{4} (\mu_1(x_1,z_1) + \mu_2(x_1,z_1)) (\mu_1(x_2,z_2)\mu_1(x_3,z_3) + \mu_2(x_2,z_2) \mu_2(x_3,z_3)) + \ldots \\
    &= \frac{1}{2} \left( \mu_1(x_1,z_1) \mu_1(x_2,z_2) \mu_1(x_3, z_3) + \mu_2(x_1,z_1) \mu_2(x_2,z_2) \mu_2(x_3, z_3) \right),
\end{align*}
after canceling out cross-context multiplied terms.

\subsection{Proof of Corollary \ref{corollary:sample_upper_bound_m2}}
With Lemma \ref{lemma:third_moment_matching_m2}, we can directly verify that all trajectories in $\Eps_l$ defined with $d=2$ satisfies the moment closeness condition \eqref{eq:delta_def} up to degree $3$ with $\delta_l = O \left(\sqrt{\iota_c / n_l} \right)$. That is, we only need to explore up to second-order moments of the system, but the guarantee on the moment-closeness can be given up to the third-order degree. Thus, we can invoke Lemma \ref{lemma:eventwise_total_bound} with $d = 3$, and combine that with Lemma \ref{lemma:reward_free_guarantee} with $d = 2$, which gives
\begin{align*}
    \|(\PP^{(1)}_\pi - \PP^{(2)}_\pi) (\tau)\|_1 &\le \sum_{l = 0}^{L+1} \sup_{\pi \in \Pi} \PP^{(1)}_\pi( x_{1:H} \in \Eps_l' ) \cdot O(\epsilon_l) \\
    &\le \sum_{l = 0}^{L+1} H^2 \epsilon_{\rm pe} \sqrt{n_l/\iota_c} \cdot (4HZ)^{3} \cdot O\left(\sqrt{\iota_c/n_l}\right) \\
    &\le O(L) H^2 \epsilon_{\rm pe} (4HZ)^3,
\end{align*}
where $\epsilon_{\rm pe} = \epsilon / (H^6 L (4Z)^3)$. Plugging this to the first part of Lemma \ref{lemma:reward_free_guarantee}, after $K$ exploration episodes where
\begin{align*}
    K \ge C \cdot \frac{(SA)^2}{\epsilon_{\rm pe}^2} \log(K/\eta) = C \cdot \frac{(SA)^2}{\epsilon^2} \cdot \poly(H, Z) \cdot \poly \log(K/\eta),  
\end{align*}
we obtain an $O(\epsilon)$-optimal policy.

\subsection{Proof of Lemma \ref{lemma:eventwise_total_bound_discrete}}
\label{appendix:proof_main_lemma_discrete_rewards}
The proofs here are largely adapted from \cite{chen2019beyond} (see their Lemma 3.1 and 3.8 for the proof of distributional identifiability from low-degree moments). We first define some notation.

We often use a single letter $y$ to denote a pair of state-action and reward $(x,z)$, and thus we use $\mu_m(y) = \mu_m(x,z)$. $\mY_q$ be a $q$ power set of state-action-rewards:
\begin{align*}
    \mY_q := \{\bm{y} = (y_1, \ldots, y_q)| (y_1, \ldots, y_q) \in (\mS \times \mA \times \mZ)^{\bigotimes q}\}.
\end{align*}
Let $\mY_0 = \{\emptyset\}$ be a null sequence set, and let $\mY := \cup_{q=1}^H \mY_q$ be a set of at most length $H$ sequence (with possible repetitions) of state-action-rewards. Then we define a latent moment matrix $\mathbf{M} \in \mathbb{R}^{|\mY| \times M}$ whose rows are indexed by $\bm{y} \in \mY$ such that 
\begin{align*}
    \mathbf{M}(\bm{y}, m) := \Pi_{q=1}^{|\bm{y}|} \mu_m(y_q).
\end{align*}
By convention, $\mathbf{M}(\emptyset, m) = 1$. For any $Y \subseteq \mY$, let $\mathbf{M}_{Y}$ be a row restriction of $\mathbf{M}$ to $Y$. We also denote a single row vector $\mathbf{M}_{\bm{y}}$ indexed by $\bm{y}$. We denote a length of sequence $\bm{y}$ as $|\bm{y}|$. For any $J \subseteq [|\bm{y}|]$, $\bm{y}_J$ is a subsequence of $\bm{y}$ restricted to $J$. If $J = \emptyset$, then $\bm{y}_{J}$ means $\emptyset$.

Now for two RMMDP models $\mM^{(1)}$ and $\mM{(2)}$ with the same transition and initial state probabilities, let $\mathbf{M}^{(1)}$ and $\mathbf{M}^{(2)}$ be latent moment matrices respectively, and let $\overline{\mathbf{M}} \in \mathbb{R}^{|\mY| \times 2M}$ be a column-concatenation of two matrices $\overline{\mathbf{M}} = \left[\mathbf{M}^{(1)} | \mathbf{M}^{(2)} \right]$. We first show that for any row of $\overline{\mathbf{M}}$ corresponding to a sequence $\bm{y}$ of length larger than $d = O(\log M)$, $\overline{\mathbf{M}}_{\bm{y}}$ is in the row span of $\overline{\mathbf{M}}_{Y(\bm{y})}$ where $Y(\bm{y})$ is a set of at most $d$ pairs in $\bm{y}$
\begin{align*}
    Y(\bm{y}) := \{ \bm{y}_{J} | \forall J \subseteq \left[ |\bm{y}| \right]: |J| \le d \}.
\end{align*}
Formally, we show the following lemma:
\begin{lemma}
    \label{lemma:exact_idneitifability_discrete}
    For any $\bm{y} \in \mY$ with $|\bm{y}| > d = \lceil 2P \log M \rceil$, the rows of $\overline{M}_{Y(\bm{y})}$ span all rows in $\overline{M}_{\bm{y}}$. 
\end{lemma}
The proof of Lemma \ref{lemma:exact_idneitifability_discrete} is deferred to Section \ref{appendix:exact_idneitifability_discrete}. The implication of Lemma \ref{lemma:exact_idneitifability_discrete} is crucial: it implies that if we can match up to all degree $d = O(\log M)$ moments exactly, then we can predict probabilities of arbitrary length of trajectories exactly. This means two RMMDP models are identical in terms of trajectory distributions. Of course, we always have a sampling noise in our estimates, and the main challenge is to understand how much the overall statistical error is amplified. 

We first observe that the statistical distance between two RMMDP models for any history-dependent policy $\pi$ can be represented as the following:
\begin{align*}
    \sum_{\tau: x_{1:H} \in \Eps_l'} | \PP^{(1)}_{\pi} (\tau) - \PP^{(2)}_{\pi} (\tau) | = \left\| \mathbf{D} \overline{\mathbf{M}}_{\mY_H} w \right\|_1,
\end{align*}
where $\mathbf{D}$ is a diagonal matrix whose diagonal element is defined as:
\begin{align*}
    \mathbf{D} \left(\bm{y} = (s,a,r)_{1:H} \right) =  \indic{\Eps_l'} \pi_{1:H} T_{1:H},
\end{align*}
and $w \in \mathbb{R}^{2M}$ is a vector concatenating $w^{(1)}$ and $-w^{(2)}$ such that $w := \left[w^{(1)} | -w^{(2)} \right]^\top$. Let $\mathbf{N}$ be a row restriction of $\overline{\mathbf{M}}$ to pairs of degree $\le d$ that are explored in $\Eps_l$, {\it i.e.,} $\mathbf{N} := \overline{\mathbf{M}}_{\overline{\mY}_l}$ where
\begin{align*}
    \overline{\mY}_l := \cup_{\tau \in \Eps_l} Y(\tau).
\end{align*}
Here we consider $\tau$ in the form of $\bm{y}$ of length $H$. By the moment closeness condition, note that 
\begin{align*}
    \|\mathbf{N} w\|_{\infty} \le \delta_l. 
\end{align*}

The remaining steps follow the proof of Lemma 3.8 in \cite{chen2019beyond}, and again we rewrite the major procedures for the completeness of the paper. We show by contradiction that if $\left\| \mathbf{D} \overline{\mathbf{M}}_{\mY_H} w \right\|_1 > \PP_{\Eps_l'}^*(\phi) \cdot \epsilon$, then it must hold that $\|\mathbf{N} w\|_{\infty} > M^{-O(M)} \cdot \epsilon$. This concludes Lemma \ref{lemma:eventwise_total_bound_discrete} by plugging $\epsilon = M^{O(M)} \cdot \delta_l$. 

To show this, let $r := \mathrm{rank} (\mathbf{N})$ be the rank of $\mathbf{N}$, and let $\mathbf{N}_r$ be the column restriction of $\mathbf{N}$ to $r$ linearly independent columns. Since the columns of $\mathbf{N}_r$ span all columns of $\mathbf{N}$, we can find $w_r := w + v$ such that $v \in \mathrm{ker}(\mathbf{N})$ and $w_r$ is only supported on the $r$ coordinates corresponding to columns selected by $\mathbf{N}_r$. Since $\mathbf{N}_r$ is full rank, $\sigma_{min}^{\infty} (\mathbf{N}_r) > 0$ where $\sigma_{min}^{\infty} (A) := \min_{u} \|Au\|_{\infty}/\|u\|_{\infty}$ for a matrix $A$. If we can give proper lower bounds for $\sigma_{min}^{\infty} (\mathbf{N}_r)$ and $\|w_r\|_{\infty}$, then we can bound $\|\mathbf{N} w\|_{\infty} \ge \sigma_{min}^{\infty}(\mathbf{N}_r) \cdot \|w_r\|_{\infty}$. Now this follows from the two following lemmas.
\begin{lemma}
    \label{lemma:min_singular_value}
    If a matrix $A \in \mathbb{R}^{n \times k}$ is a full column rank with $n \gg k$, and if all elements of $A$ are integral multiples of some $p > 0$, then $\sigma_{min}^{\infty} (A) \ge p^k \cdot k^{-O(k)}$. 
\end{lemma}
Note that all entries of $\mathbf{N}_r$ are integral multiples of $1/P^{d}$, and thus we have that 
\begin{align*}
    \sigma_{min}^\infty (\mathbf{N}_r) \ge P^{-d r} r^{-O(r)} \ge M^{-O(M P \log P)}.
\end{align*} 
Since $P = O(1)$, this is bounded below by $M^{-O(M)}$. The proof of Lemma \ref{lemma:min_singular_value} is given in Section \ref{appendix:min_singular_value}. On the other hand, we can show that $\|w_r\|_{\infty} > \epsilon / M$.
\begin{lemma}
    \label{lemma:bound_on_weights}
    If $\left\| \mathbf{D} \overline{\mathbf{M}}_{\mY_H} w \right\|_1 > \PP_{\Eps_l'}^*(\phi) \cdot \epsilon$, then for any $v \in \mathrm{ker}(\mathbf{N})$, $\|w + v\|_{\infty} > \epsilon / (2M)$. 
\end{lemma}
\begin{proof}
    Let $w_r = w + v$. Note that by Lemma \ref{lemma:exact_idneitifability_discrete}, all rows of $\mathbf{D} \overline{\mathbf{M}}_{\mY_H}$ are spanned by the rows of $\mathbf{N}$: for any $\tau \in \Eps_l'$, then $\overline{\mathbf{M}}_{\tau}$ is in the span of $\overline{\mathbf{M}}_{Y(\tau)}$ and thus spanned by the rows of $\mathbf{N}$, and otherwise $\mathbf{D}(\tau) = 0$ by definition and thus the row of $\mathbf{D} \overline{\mathbf{M}}_{\mY_H}$ corresponding to $\tau$ is 0. Obviously, $0$ vector is in the span of the rows of $\mathbf{N}$. Now since $v \in \mathrm{ker}(\mathbf{N})$, $\mathbf{D} \overline{\mathbf{M}}_{\mY_H} v = 0$ for any $v \in \mathrm{ker}(\mathbf{N})$. Therefore we have
    \begin{align*}
        \left\| \mathbf{D} \overline{\mathbf{M}}_{\mY_H} w \right\|_1 = \left\| \mathbf{D} \overline{\mathbf{M}}_{\mY_H} w_r \right\|_1 \le \left\| \mathbf{D} \overline{\mathbf{M}}_{\mY_H} \right\|_{1,1} \|w_r\|_{\infty},
    \end{align*}
    where $\left\| \mathbf{D} \overline{\mathbf{M}}_{\mY_H} \right\|_{1,1}$ is the absolute sum of all elements in $\mathbf{D} \overline{\mathbf{M}}_{\mY_H}$. Note that the sum of the $m^{th}$ column of $\mathbf{D} \overline{\mathbf{M}}_{\mY_H}$ is equal to 
    \begin{align*}
        \sum_{(x,r)_{1:H}} \indic{\Eps_l'} \pi_{1:H} T_{1:H} \Pi_{t=1}^H \mu_m^{(1)}(x_t ,r_t) \le \PP_{\Eps_l'}^* (\phi),
    \end{align*}
    and similar inequalities hold for the $(M+m)^{th}$ column for $m \in [M]$. Since there are $2M$ columns, $\left\| \mathbf{D} \overline{\mathbf{M}}_{\mY_H} \right\|_{1,1}$ is further bounded by $2M \PP_{\Eps_l'}^* (\phi)$. Finally, by a contradicting assumption, we have
    \begin{align*}
        \epsilon \PP^*_{\Eps_l'} (\phi) < 2M \PP^*_{\Eps_l'}(\phi) \|w_r\|_{\infty},
    \end{align*}
    which proves the lemma. 
\end{proof}
\vskip 0.2cm

Combining Lemma \ref{lemma:min_singular_value} and \ref{lemma:bound_on_weights}, we obtain the desired contradiction that $\|\mathbf{N} w\|_{\infty} > M^{-O(M)} \epsilon$. By letting $\epsilon = M^{O(M)} \delta_l$, we can conclude that $\left\| \mathbf{D} \overline{\mathbf{M}}_{\mY_H} w \right\|_1 \le \PP_{\Eps_l'}^*(\phi) \cdot M^{O(M)} \delta_l$, and we can conclude that
\begin{align*}
    \sum_{\tau: x_{1:H} \in \Eps_l'} | \PP^{(1)}_{\pi} (\tau) - \PP^{(2)}_{\pi} (\tau) | = \left\| \mathbf{D} \overline{\mathbf{M}}_{\mY_H} w \right\|_1 \le \PP_{\Eps_l'}^* (\phi) \cdot M^{O(M)} \delta_l.
\end{align*}

\subsubsection{Proof of Lemma \ref{lemma:exact_idneitifability_discrete}}
\label{appendix:exact_idneitifability_discrete}
    This largely follows from the proof of Lemma 3.1 in \cite{chen2019beyond}, and we rewrite the major procedures in there for the completeness of the paper. We show this lemma by mathematical induction on the length of sequence. For convenience, let $n = |\bm{y}|$ be the length of target sequence. We show that there exists non-trivial coefficients $\{\alpha_{J}\}_{J \subseteq [n]}$ such that
    \begin{align*}
        \sum_{J \subseteq [n]} \alpha_{J} \overline{\mathbf{M}}_{\bm{y}_J} = 0,
    \end{align*}
    and that $\alpha_{[n]}$ is nonzero. If we can do this inductively from $n = d+1$, then we are done by mathematical induction. We construct an auxiliary polynomial function $f$ of $n$ variables $x = \{x_j\}_{j=1}^n$ such that:
    \begin{align*}
        f( x ) = \Pi_{j=1}^n (x_j - \lambda_j) = \sum_{J \subseteq [n]} \alpha_J \Pi_{j \in J} x_j, 
    \end{align*}
    for some $\{\lambda_j\}_{j=1}^n$. Note that the coefficient  $\alpha_{[n]}$ is always $1$. The strategy is to construct a polynomial $f$ such that $f(x) = 0$ at all $x = \{\mu_m^{(b)}(y_j)\}_{j=1}^n$ for all $m \in [M]$ and $b = 1,2$. Note that any column of $\sum_{J\subseteq[n]} \alpha_J \overline{\mathbf{M}}_{\bm{y}_J}$ corresponds to one of $\{\mu_m^{(b)}(y_j)\}_{j=1}^n$. The existence of such polynomial $f$ guarantees that $\overline{\mathbf{M}}_{\bm{y}}$ is in the span of the rows of lower degree pairs in the same sequence, which inductively implies the lemma.
    
    To construct $f$, we start with $f_0(x) = 1$ at $t=0$ and inductively construct $f_{t+1}$ from $f_{t}$ where $f_t(x) = \Pi_{j=1}^t (x_j - \lambda_j)$. At any time step $t$, define a set of surviving columns $R_t = \left\{(b,m) \big| f_{t} \left( \{\mu_m^{(b)}(y_j)\}_{j=1}^n \right) \neq 0 \right\}$.
    Since $\mu_m^{(b)} (\cdot)$ can take values only from the candidate probability set $\mP$, by the pigeonhole principle, we can choose $\lambda_{t+1} \in \mP$ such that $|R_{t+1}| \le \left\lfloor \left(1 - \frac{1}{P+1}\right) |R_t| \right\rfloor$. Since $|R_0| = 2M$, once $t$ reach $(P+1) \log (2M)$, there will be no surviving columns and we find the desired polynomial $f = f_t(x) \cdot \Pi_{j=t+1}^n x_j$.

\subsubsection{Proof of Lemma \ref{lemma:min_singular_value}}
\label{appendix:min_singular_value}
    This is reminiscent of Lemma 3.7 in \cite{chen2019beyond}. We can pick $k$ rows of $A$ such that a row restriction of $A$ to the selected rows, which we denote as $A_k$, is full rank. By definition, $\sigma_{min}^{\infty,1}(A) \ge \sigma_{min}^{\infty}(A_k)$. Now $A_k$ is a $k\times k$ square matrix and $\mathrm{det}(A_k) > 0$, and thus we can equivalently say
    \begin{align*}
        \sigma_{min}^{\infty}(A_k) = \min_{u} \frac{\|A_k u\|_{\infty}}{\|u\|_\infty} = \min_{u'} \frac{\|u'\|_{\infty}}{\|A_k^{-1} u'\|_\infty} \ge \min_{u'} \frac{\|u'\|_{\infty}}{\|A_k^{-1}\|_{\infty,\infty} \|u'\|_1} \ge \frac{1}{k} \frac{1}{\|A_k^{-1}\|_{\infty,\infty}},
    \end{align*}
    where $\|A_k^{-1}\|_{\infty, \infty}$ is the largest element of $A_k^{-1}$. The determinant of any $(k-1) \times (k-1)$ minor is at most $(k-1)! = k^{O(k)}$, and $\mathrm{det}(A_k)$ is some nonzero integral multiple of $p^k$. Using the Cramer's matrix inversion formula, we can conclude that $\|A_k^{-1}\|_{\infty, \infty} \le p^{-k} k^{O(k)}$.

\subsection{Proof of Theorem \ref{theorem:sample_complexity_upper_bound_discrete}}
Similarly to the proof of Theorem \ref{theorem:sample_complexity_upper_bound} (which can be found in Section \ref{appendix:proof:sample_complexity_upper_bound}), we can show that
\begin{align*}
    \|(\PP^{(1)}_\pi - \PP^{(2)}_\pi) (\tau)\|_1 &= \sum_{l=0}^{L+1} \sum_{\tau:x_{1:H} \in \Eps_l'} |\PP^{(1)}_{\pi}(\tau) - \PP^{(2)}_{\pi}(\tau) | \\
    &\le \sum_{l = 0}^{L+1} H^d \epsilon_{\rm pe} \cdot \sqrt{n_l / \iota_c} \cdot M^{O(MP \log P)} \cdot O\left( \sqrt{\iota_c / n_l} \right) \\
    &\le L H^d \epsilon_{\rm pe} M^{O(M)},
\end{align*}
where we used Lemma \ref{lemma:eventwise_total_bound_discrete} and Lemma \ref{lemma:reward_free_guarantee} with $d = 2P \log M$ since $H > 2P \log M$. Plugging $\epsilon_{\rm pe} = \epsilon / (H^{d+1} L M^{O(M)})$, we have 
\begin{align*}
    |V_{\mM^{(1)}}^{\pi} - V_{\mM^{(2)}}^{\pi}| \le H \cdot \|(\PP^{(1)}_{\pi} - \PP^{(2)}_{\pi}) (\tau)\|_1 \le O(\epsilon).
\end{align*}

\section{Appendix for Pure-Exploration}
\label{appendix:pure_exploration_algo}

\begin{algorithm}[t]
    \caption{Pure Exploration of Higher-Order Moments}
    \label{algo:pure_explore}
    \begin{algorithmic}[1]
        \STATE{{\bf Function:} \texttt{EstimateMoments}($d, \epsilon, \eta$)} 
        \STATE{Let $\epsilon_{\rm pe}:= \epsilon/(HL(4H^2Z)^d))$ }
        \STATE{Initialize $\widetilde{Q}_{(\cdot)} (\cdot) = \widetilde{V}_0 = 1$, $n(\bm{x}) = 0$ for all $\bm{x} \in \bigcup_{q=1}^d (\mS \times \mA)^{\bigotimes q}$.}
        \STATE{Initialize $\hat{T}(\cdot)=\hat{\nu}(\cdot)=0$, $n_T(\cdot) = 0$.}
        \WHILE{$\widetilde{V}_0 > \epsilon_{pe}$}
            \STATE {Get an initial state $s_1$ for the $k^{th}$ episode. Let $\bm{v}_1 = (\emptyset,\ldots,\emptyset)$, $i = 1$, $r_{c} = 1$}
            \FOR{$t = 1, 2, ..., H$}
                \STATE {Pick $(a_t, b_t) = \arg \max_{(a,b) \in \widetilde{\mA}} \widetilde{Q}_t ((i_t, \bm{v}_t, s_t), (a,b))$.}
                \STATE{Play action $a_t$, observe next state $s_{t+1}$ and reward $r_t$.}
                \IF {$i_t \le d$ \AND $b_t \neq 0$}
                    \STATE{Record $z_{i_t} = r_t$}
                \ENDIF
                \STATE{Update $(i_{t+1}, \bm{v}_{t+1}, s_{t+1})$ according to the choice of $a_t$ and $b_t$ following the rule in \eqref{eq:augment_rule}}
                \IF{$i_t \le d$ \AND $i_{t+1} = d+1$}
                    \STATE{$\bm{x}_c:= (v_{t+1}^j)_{j=1}^{i_t}$, $\bm{z}_c = (z_j)_{j=1}^{i_t}$}
                \ENDIF
            \ENDFOR
            \IF{$i_{H+1} = d+1$}
                \STATE{$n(\bm{x}_c) \leftarrow n(\bm{x}_c) + 1$}
                \STATE{$\moment_n (\bm{x}_c, \bm{z}) \leftarrow \left(1 - 1/n(\bm{x}_c) \right) \moment_n (\bm{x}_c, \bm{z}) + \indic{\bm{z} = \bm{z}_c} / n(\bm{x}_c)$ for all $\bm{z} \in \mZ^{\bigotimes {\rm length} (\bm{x})}$}
            \ENDIF
            \STATE{Update $\hat{T}, \hat{\nu}$ and $n_T$ from a trajectory $(s_1,a_1, s_2,a_2,...,s_H,a_H)$ \label{line:update_transition}}
            \STATE{Update $\widetilde{Q}_{(\cdot)}$ and $\widetilde{V}$ using \eqref{eq:update_q2}, \eqref{eq:optimistic_value}}
        \ENDWHILE
        \STATE{{\bf Return} $\moment_n, n$}
    \end{algorithmic}
\end{algorithm}

This part mostly follows \cite{kwon2021reinforcement}, and we may repeat most of the steps for the completeness of the paper. We employ the reward-free exploration idea for $d^{th}$-order MDPs, which is defined as the following:
\begin{definition}[$d^{th}$-Order MDPs]
    \label{definition:dth_order_mdp}
    A $d^{th}$-order MDP $\widetilde{\mM}$ is defined on a state-space $\widetilde{\mS}$ and action-space $\widetilde{\mA}$ where 
    \begin{align*}
        \widetilde{\mS} &= \left\{ (i, \bm{v}, s)| \ i \in [d+1] , \bm{v}: (v^i)_{i=1}^d \in ((\mS \times \mA) \cup \{\emptyset\})^{\bigotimes d}, s \in \mS \right\}, \\
        \widetilde{\mA} &= \{ (a, b)| \ a \in \mA, b \in \{0,1,-1\} \}.
    \end{align*}
    In $\widetilde{\mM}$, an augmented state $\left( i_t, \bm{v}_t, s_t \right)$ evolves under an action $(a_t, z_t)$ as follows:
    \begin{align}
        &i_1 = 1, \bm{v}_1 = (\emptyset,\ldots,\emptyset), s_1 \sim \nu(\cdot), \quad s_{t+1} \sim T(\cdot | s_t,a_t), \quad v^j_{t+1} = v^j_t \ \ \forall j \in [d] / \{i_t\}, \nonumber \\
        &i_{t+1} = \begin{cases} d + 1 & \text{if } b_t = -1 \\ i_t+1 & \text{else if } b_t = 1 \text{ and } i_t \le d \\ i_t & \text{else} \end{cases}, \quad v^{i_t}_{t+1} = \begin{cases} (s_t, a_t) & \text{if } b_t \neq 0 \\ v^{i_t}_t & \text{else } \end{cases} \ \text{when } i_t \le d. \label{eq:augment_rule}
    \end{align}
\end{definition}
In short, additional state variables $i$ and $v$ select which state-actions to include in a moment to estimate in the current episode. Additional action variable $b$ selects whether to include or skip the current state-action, or decide a moment to sample with currently saved state-actions in $v$. 

Let us define the upper confidence action-value function $\widetilde{Q}$ and value function $\widetilde{V}$ that is defined as in the form of Bellman-equation w.r.t. $\widetilde{\mM}$ with pure-exploration bonus:
\begin{align}
    \label{eq:update_q2}
    &q_c((i, \bm{v}, s), (a, b)) = \indic{(i\le d) \cap (i'=d+1)} \cdot \left(1 \wedge \sqrt{\frac{\iota_c}{n(\bm{x}_c)}} \right), q_T(s,a) = \left(1 \wedge \sqrt{\frac{\iota_T}{n_T(s,a)}}\right) \nonumber \\
    &\widetilde{Q}_t ((i, \bm{v},s), (a,b)) = 1 \wedge \left(q_c ((i, \bm{v}, s), (a, b)) + \Exs_{s' \sim \hat{T}(\cdot|s,a)} \left[\widetilde{V}_{t+1} (i', \bm{v}',s') \right]  + q_T(s,a) \right),
\end{align}
and
\begin{align}
    \widetilde{V}_t(i, \bm{v},s) = \max_{(a',b') \in \widetilde{\mA}} \widetilde{Q}_t((i, \bm{v}, s), (a',b')), \quad \widetilde{V}_0 = \sqrt{\iota_\nu/k} + \tssum_{s} \hat{\nu} (s) \cdot \widetilde{V}_1(1, \bm{v}_1, s),  \label{eq:optimistic_value}
\end{align}
Here, $i'$ and $\bm{v}'$ are the first and second coordinates of the next state following the (deterministic) transition rule \eqref{eq:augment_rule} for $i$ and $\bm{v}$. $\indic{(i\le d) \cap (i'=d+1)}$ is an indicator of whether to finish and collect samples for correlations stored in $v'$. By convention, we let $\widetilde{Q}_{H+1}(\cdot) = 0$. $K$ is the total number of episodes to be explored. The logarithmic factor $\iota_c = O(d \log(SAZ/\eta))$ is properly set confidence interval parameters. The pure-exploration bonus $q_c$ encourages to collect samples for the moments that have not been sufficiently explored yet. This is controlled by the number $n(\bm{x}_c)$ of collected samples for $\bm{x}_c = ( (v^{j})_{j=1}^{i-1}, (s,a) )$. Variables $n_T, q_T$ are defined for estimating transition models which we describe below, where $n_T(s,a)$ is the number of total times that $(s,a)$ has been visited, and $q_T(s,a)$ is pure-exploration bonus for visiting $(s,a)$.

\paragraph{Exploration for Estimating Moments} In every episode, we take a greedy augmented action $(a_t, b_t)$ that maximizes $\widetilde{Q}_t$ at every time step $t \in [H]$. We continue this pure-exploration process for $K$ episodes until $\widetilde{V}_0 \le \epsilon_{\rm pe}$ with a threshold parameter $\epsilon_{\rm pe}$ for the pure exploration. The pure-exploration procedure is summarized in Algorithm \ref{algo:pure_explore}. The main purpose of Algorithm \ref{algo:pure_explore} is to auto-balance the amount of samples for moments proportional to each moment's reachability. 

\paragraph{Estimate Transition Models} 
The transition models and initial state distributions can be easily estimated in the pure-exploration phase. In equation \eqref{eq:update_q2}, $q_T(\cdot)$ is an exploration bonus term for the uncertainty in transition probabilities, and $\iota_T = O(S \log(K/\eta))$ and $\iota_\nu = O(S \log(K/\eta))$ are properly set confidence constants. Specifically, we can add bonus terms $q_T(\cdot)$ and $\sqrt{\iota_\nu / k}$ to upper-confidence functions $\widetilde{Q}$ and $\widetilde{V}_0$ to encourage the exploration to estimate transition model $\hat{T}$ and initial state distribution $\hat{\nu}$ simultaneously with higher-order moments of latent reward models. The update step (line \ref{line:update_transition}) can be implemented in a straight-forward manner.

\paragraph{Additional Notation} We denote $\mathds{1}_k \{\bm{x}\}$ as a random variable indicating whether $\bm{x}$ is visited at the $k^{th}$ episode. Let $\widetilde{\Pi}:\widetilde{\mS} \rightarrow \widetilde{\mA}$ be the class of {\it stationary} policies in $d^{th}$-order MDPs. $\widetilde{\pi}_k$ be the policy (greedy with respect to $\widetilde{Q}$) executed in the $k^{th}$ episode. Let $n_k(\bm{x}) := \sum_{k'=1}^{k-1} \mathds{1}_{k'} \{\bm{x}\}$ and the expected quantities $\bar{n}_k (\bm{x}) := \sum_{k'=1}^{k-1} \Exs^{\widetilde{\pi}_{k'}} [\mathds{1}_{k'} \{\bm{x}\}]$. We define a desired high probability event $\Eps_{pe}$ for martingale sums:
\begin{align}
    \label{eq:def_martingale_event}
    n_k(\bm{x}) \ge \frac{1}{2} \bar{n}_k (\bm{x}) - c_l \cdot d \log (SAK/\eta), &\qquad \forall k \in [K], \bm{x} \in \textstyle \bigcup_{q=1}^d (\mS \times \mA)^{\bigotimes q},
\end{align}
for some absolute constant $c_l > 0$. With a standard measure of concentration argument for martingale sums~\cite{wainwright2019high}, and taking union bound on all $k$ and $\bm{x}$, we can show that $\PP(\Eps_{pe}) \ge 1 - \eta$. 

We denote $\hat{T}_k, \hat{\nu}_k$ for the empirically estimated transition and initial distribution models at the beginning of $k^{th}$ episode. Similarly to $n_k(\bm{x})$, let $n_k(s,a), \bar{n}_k(s,a)$ be the actual and expected visit count for a single state-action $(s,a) \in (\mS \times \mA)$, and let $\#_k (s,a)$ be the random variable that the number of times that $(s,a)$ is visited at the $k^{th}$ episode. This is for tracking the uncertainties in $\hat{T}_k$. Similarly to equation \eqref{eq:def_martingale_event}, it holds that
\begin{align*}
    n_k(s,a) \ge \frac{1}{2} \bar{n}_k (s,a) - c_l \cdot \log (SAK/\eta), &\qquad \forall k \in [K], (s,a) \in (\mS \times \mA),
\end{align*}
with probability at least $1 - \eta$.

\subsection{Proof of Lemma \ref{lemma:reward_free_guarantee}}
\label{appendix:reward_free_guarantee}
\newcommand{\wpi}{\widetilde{\pi}}

\paragraph{Proof of equation \eqref{eq:required_episodes}:} We first show that Algorithm \ref{algo:pure_explore} terminates after at most $K$ episodes with probability at least $1 - \eta$ where
\begin{align*}
    K \ge C \cdot (SA)^d \epsilon_{pe}^{-2} \cdot \log (K/\eta),
\end{align*}
for some absolute constant $C > 0$. Let us examine $\widetilde{V}_0$ at the $k^{th}$ episode. This can be decomposed as
\begin{align*}
    \widetilde{V}_0 &= \sqrt{\iota_\nu / k} + \sum_{s} \hat{\nu}_k (s) \cdot \widetilde{V}_1 (i_1, \bm{v}_1, s) \\
    &\le \sqrt{\iota_\nu / k} + \| \hat{\nu}_k (s) - \nu(s) \|_1 + \sum_{s} \nu(s) \cdot \max_{(a,b) \in \widetilde{\mA}} \widetilde{Q}_1 ((i_1,\bm{v}_1, s), (a,b)) \\
    &\le 2 \sqrt{\iota_\nu / k} + \Exs_{\wpi_{k}} \left[\widetilde{Q}_1 ((i_1, \bm{v}_1, s_1), \widetilde{\pi}_k(i_1, \bm{v}_1,s_1)) \right],
\end{align*}
where in the last inequality we used $\|\hat{\nu}_k - \nu(s)\|_1 \le \sqrt{\iota_\nu / k}$ by standard martingale inequalities. Then, we can recursively bound expectation of $\widetilde{Q}_t$ for $t \ge 1$. For convenience, let us denote $\bm{x}_t = ((v_t^{j})_{j=1}^{i_t-1}, (s_t,a_t))$ be the moment that can be sampled at the current time step, and
\begin{align*}
    \Exs_{\wpi_{k}} &\left[\widetilde{Q}_t ((i_t, \bm{v}_t, s_t), (a_t,b_t)) \right] \\
    &= \Exs_{\wpi_{k}} \left[q_r((i_t, \bm{v}_t, s_t), (a_t,b_t))  + q_T(s_t,a_t) \right] + \Exs_{\wpi_k} \left[\tssum_{s_{t+1}} \hat{T}_k(s_{t+1} | s_t,a_t) \cdot \widetilde{V}_{t+1} (i_{t+1}, \bm{v}_{t+1}, s_{t+1}) \right] \\
    &\le \Exs_{\wpi_{k}} \left[\widetilde{Q}_{t+1} ((i_{t+1}, \bm{v}_{t+1}, s_{t+1}), (a_{t+1},b_{t+1})) \right] \\
    &\qquad + 2 \Exs_{\wpi_{k}} \left[\indic{i_t \le d \cap i_{t+1} = d+1} \left( 1 \wedge \sqrt{\frac{\iota_c}{n_k(\bm{x}_t)}} \right) \right] + 2 \Exs_{\wpi_k} \left[1 \wedge \sqrt{\frac{\iota_T}{n_k (s_t,a_t)}} \right],
\end{align*}
where in the last inequality, we used that $\|T(\cdot|s_t,a_t) - \hat{T}(\cdot | s_t,a_t)\|_1 \le \sqrt{\iota_T / n_k(s_t,a_t)}$ by martingale concentration, and $|\widetilde{Q}_{t+1}(\cdot)| \le 1$. Note that the indicator $\indic{i_t \le d \cap i_{t+1} = d+1}$ means whether we collect the sample at the $t^{th}$ time step, {\it i.e.,} $\indic{\textrm{collect at $t$}}$. Putting together, at the $k^{th}$ episode, we have
\begin{align*}
    \widetilde{V}_0 &\le 2\sqrt{\iota_\nu / k} + 2 \sum_{t=1}^H \Exs_{\wpi_k} \left[ \left(1 \wedge \sqrt{\frac{\iota_T}{n_k (s_t,a_t)}} \right) + \indic{\textrm{collect at $t$}} \left(1 \wedge \sqrt{\frac{\iota_c}{n_k (\bm{x}_{t})}}\right) \right] \\
    &= 2\sqrt{\iota_\nu / k} + 2 \sum_{(s,a)} \left(1 \wedge \sqrt{\frac{\iota_T}{n_k (s,a)}} \right) \cdot \Exs_{\wpi_k} \left[ \#_k (s,a) \right] + 2 \sum_{\bm{x}} \left(1 \wedge \sqrt{\frac{\iota_c}{n_k (\bm{x})}} \right) \cdot \Exs_{\wpi_k} \left[\mathds{1}_k \{\bm{x}\} \right].
\end{align*}
From equation \eqref{eq:def_martingale_event}, we have that
\begin{align*}
    \sum_{\bm{x}} \left(1 \wedge \sqrt{\frac{\iota_c}{n_k (\bm{x})}} \right) \cdot \Exs_{\wpi_k} \left[\mathds{1}_k \{\bm{x}\} \right] &\le 2\sum_{\bm{x}} \sqrt{\frac{\iota_c}{1 + \bar{n}_k (\bm{x})}} \cdot (\bar{n}_{k+1}(\bm{x}) - \bar{n}_k (\bm{x})) \\
    &\quad + \sum_{\bm{x}} \indic{\bar{n}_k (\bm{x}) < 4 \cdot c_l d \log(SAK/\eta) } (\bar{n}_{k+1}(\bm{x}) - \bar{n}_k (\bm{x})),
 \end{align*}
where we used by definition that $\Exs_{\wpi_k} \left[\mathds{1}_k \{\bm{x}\} \right]=(\bar{n}_{k+1}(\bm{x}) - \bar{n}_k (\bm{x}))$. Similarly, we also have
\begin{align*}
    \sum_{(s,a)} \left(1 \wedge \sqrt{\frac{\iota_T}{n_k (s,a)}} \right) \cdot \Exs_{\wpi_k} \left[\#_k (s,a) \right] &\le 2 \sum_{s,a} \sqrt{\frac{\iota_T}{1 + \bar{n}_k (s,a)}} \cdot (\bar{n}_{k+1}(s,a) - \bar{n}_k (s,a)) \\
    &+ H \sum_{(s,a)} \indic{\bar{n}_k (s,a) < 4 \cdot c_l \log(SAK/\eta) } (\bar{n}_{k+1}(s,a) - \bar{n}_k (s,a)).
\end{align*}
using an integral inequality $\sum_{k=1}^K \sqrt{1/(1+n_k)} (n_{k+1} - n_k) \le \int_{1}^{n_K} 1/x dx$ for any non-decreasing sequence $(n_k)_{k=1}^K$, we can sum over all $K$ episodes until $\widetilde{V}_0 > \epsilon_{\rm pe}$ and thus, we have
\begin{align*}
    K \epsilon_{pe} \le 4 \sqrt{\iota_\nu K} + O(c_l d H (SA)^d \log(KSA/\eta)) + 8 \sum_{(s,a)} \sqrt{\iota_T \bar{n}_{K+1} (s,a)} + 8 \sum_{\bm{x}} \sqrt{\iota_c \bar{n}_{K+1} (\bm{x})}.
\end{align*}
We now note that $\sum_{s,a} \bar{n}_{K+1} (s,a) = HK$ and $\sum_{\bm{x}} \bar{n}_{K+1} (\bm{x}) \le K$. Using a Cauchy-Schwartz inequality, we get
\begin{align*}
    K \epsilon_{pe} \le O \left(\sqrt{\iota_\nu K} + d H(SA)^d  \log(KSA/\eta) + \sqrt{\iota_T HSAK} + \sqrt{\iota_c (SA)^d K} \right). 
\end{align*}
The bound on $K$ is concluded by plugging the confidence parameters, which ensures that $K$ should satisfy 
\begin{align*}
    K \le O\left(Hd (SA)^d \epsilon_{pe}^{-2} \log(KSA/\eta) \right),
\end{align*}
until we terminate Algorithm \ref{algo:pure_explore} after at most $K$ episodes with probability at least $1 - \eta$.

\paragraph{Proof of equation \eqref{eq:max_event_probability}:} 
To prove this part, we first note that by union bound, we have
\begin{align*}
    \sup_{\pi\in \Pi} \PP_{\pi} (x_{1:H} \in \Eps_l') &= \sup_{\pi\in \Pi} \PP_{\pi} \left( \bigcup_{q=1}^d \bigcup_{1 \le t_1 < \ldots < t_q \le H} (x_{t_i})_{i=1}^q \in \mX_l \cap \mX_{l-1}^c \right) \\
    &\le \sum_{q=1}^d \sum_{1 \le t_1 < \ldots < t_q \le H} \sup_{\pi\in \Pi} \PP_{\pi} \left( (x_{t_i})_{i=1}^q \in \mX_{l-1}^c \right).
\end{align*}
For each fixed $q$ and $\bm{t} = (t_i)_{i=1}^q$, we consider a sub-class of pure-exploration policies $\widetilde{\Pi}(\bm{t}):(\widetilde{\mS} \times [H]) \rightarrow \widetilde{\mA}$ such that each $\widetilde{\pi} \in \widetilde{\Pi}(\bm{t})$ takes $b_t = 1$ when $t = t_i$ for some $i < q$, $b_t = -1$ when $t = t_q$, and otherwise takes $b_t = 0$. Within this policy class, define the value function $\widetilde{V}^{\bm{t}}$ and action-value function $\widetilde{Q}^{\bm{t}}$ with respect to $\widetilde{\Pi}(\bm{t})$ as the following: 
\begin{align*}
    &q_T(s,a) = \left(1 \wedge \sqrt{\frac{\iota_T}{n_{K+1} (s,a)}} \right), \nonumber \\
    &q_{t,c} ((i, \bm{v}, s), (a,b)) = \indic{t=q \cap \bm{x}_c \in \mX_{l-1}^c} \cdot \left(1 \wedge \sqrt{\frac{\iota_c}{n_{K+1} (\bm{x}_c)}} \right), \nonumber \\ 
    &\widetilde{Q}_t^{\bm{t}} ((i, \bm{v},s), (a,b)) = 1 \wedge \left(q_{t,c} ((i, \bm{v}, s), (a, b)) + \Exs_{s' \sim \hat{T}(\cdot|s,a)} \left[\widetilde{V}_{t+1}^{\bm{t}} (i', \bm{v}',s') \right] + q_T(s,a) \right),
\end{align*}
and
\begin{align*}
    &\widetilde{V}_t^{\bm{t}} (i, \bm{v},s) = \max_{a \in \mA} \widetilde{Q}_t^{\bm{t}} ((i, \bm{v}, s), (a,b_t)), \quad \widetilde{V}_0^{\bm{t}} = \sqrt{\iota_\nu / K} + \tssum_{s} \hat{\nu} (s) \cdot \widetilde{V}_1^{\bm{t}} (1, \bm{v}_1, s).
\end{align*}
with $\widetilde{Q}_{H+1}^{\bm{t}} = 0$. By construction, $\widetilde{V}_0$ is an upper confidence bound of $\widetilde{V}_0^{\bm{t}}$:
\begin{align*}
    \epsilon_{\rm pe} \ge \widetilde{V}_0 \ge \widetilde{V}_0^{\bm{t}},
\end{align*}
since $\widetilde{V}_0^{\bm{t}}$ is computed with more restriction on policies. Note that the exploration-bonus from collecting a sample of moments $q_{t,c}$ is always larger than $\sqrt{\iota_c / n_{l-1}}$. On the other hand, $\sup_{\pi} \PP_{\pi} \left((x_{t_i})_{i=1}^q \in \mX_{l-1}^c \right)$ can be computed through the same dynamic programming on $\widetilde{Q}^*$ with slight changes of exploration bonus:
\begin{align*}
    &q_{t,c} ((i, \bm{v}, s), (a,b)) = \indic{t=q \cap \bm{x}_c \in \mX_{l-1}^c}, \\ 
    &\widetilde{Q}_t^* ((i, \bm{v}, s), (a,b)) = q_{t,c} + \sum_{s'} T(s'|s,a) \cdot \widetilde{V}_{t+1}^* (i', \bm{v}',s), 
\end{align*}
and
\begin{align*}
    \widetilde{V}_{t}^* (i, \bm{v},s) &= \max_{a \in \mA} \widetilde{Q}_t^*((i, \bm{v}, s),(a, b_t)).
\end{align*}
Then, 
\begin{align*}
     \widetilde{V}_0^* = \sum_{s} \nu(s) \cdot \widetilde{V}_1^*(1, \bm{v}_1, s) = \sup_{\pi \in \Pi} \PP_{\pi} \left((x_{t_i})_{i=1}^q \in \mX_{l-1}^c \right).
\end{align*}
Finally, with the setting of confidence interval parameters $\iota_T$ for transition errors, we can inductively show that
\begin{align*}
    \widetilde{Q}_t^{\bm{t}} \ge \widetilde{Q}_t^* \cdot \sqrt{\iota_c / n_{l-1}}, \quad \widetilde{V}_t^{\bm{t}} \ge \widetilde{V}_t^* \cdot \sqrt{\iota_c / n_{l-1}}.
\end{align*}
This implies that 
\begin{align*}
    \epsilon_{\rm pe} \ge \widetilde{V}_0^{\bm{t}} \ge \sqrt{\iota_c / n_{l-1}} \cdot \sup_{\pi \in \Pi} \PP_{\pi} \left((x_{t_i})_{i=1}^q \in \mX_{l-1}^c \right).
\end{align*}
We can conclude the equation \eqref{eq:max_event_probability}:
\begin{align*}
    \sup_{\pi\in \Pi} \PP_{\pi} (x_{1:H} \in \Eps_l') &\le \sum_{q=1}^d \sum_{1 \le t_1 < \ldots < t_q \le H} \sup_{\pi\in \Pi} \PP_{\pi} \left( (x_{t_i})_{i=1}^q \in \mX_{l-1}^c \right) \\
    &\le O \left(H^d \epsilon_{\rm pe} \cdot \sqrt{n_{l-1}/\iota_c} \right).
\end{align*}

\section{Proofs for the Lower Bound}
\subsection{Proof of Lemma \ref{lemma:moment_matching_lower_bound}}
This construction follows from the result in Section 4.3 in \cite{chen2019beyond}, and in particular, their Lemma 4.8 with a slight change in constants ({\it e.g.,} let $\lambda_2 = -\lambda_1 \propto -\epsilon \cdot 2^{-d})$. We refer the readers to \cite{chen2019beyond} for detailed constructions.

\subsection{Proof of Lemma \ref{lemma:optimal_policy_lower_bound}}
The optimal policy $\pi^*$ is the one which always plays optimal actions up to time step $d-1$, and select the last action depending on the conditional expectation of the last reward. Specifically, suppose we played a sequence of actions $(a^*_t)_{t=1}^{d-1}$ and the received a reward sequence $(r_t)_{t=1}^{d-1}$. It is not difficult to verify that the conditional probability of last reward according to $a_H^*$ is given as follows:
\begin{align}
    \label{eq:last_reward_conditional_prob}
    \Exs \left[r_d | (a_t^*)_{t=1}^{d-1}, (r_t)_{t=1}^{d-1}, a_d^* \right] = \begin{cases} 1/2 + \epsilon \cdot 2^{d-1} & \text{if } \sum_{t=1}^{d-1} (1 - r_t) \text{ is even} \\ 1/2 - \epsilon \cdot 2^{d-1} & \text{otherwise} \end{cases},
\end{align}
That is, the number of $0$ in a sequence $(r_t)_{t=1}^{d-1}$ is even, then the probability of getting $1$ is larger, and otherwise the probability of getting $0$ is larger. Thus, the optimal policy can play $a_H = a_H^*$ if the number of $0$ is even, and play anything else otherwise. Cumulative rewards of the optimal policy is given as follows:
\begin{align*}
    \Exs_{\pi^*} \left[\tssum_{t=1}^d r_t \right] &= \Exs_{\pi^*} \left[\tssum_{t=1}^{d-1} r_t \right] + \Exs \left[ \Exs \left[r_d \Big| (a_t)_{t=1}^{d-1} = (a_t^*)_{t=1}^{d-1}, (r_t)_{t=1}^{d-1}, a_d \sim \pi^* \right] \right] \\
    &\ge (d-1)/2 + \frac{1}{2} (1/2 + \epsilon \cdot 2^{d-1}) + \frac{1}{2} (1/2) \\
    &= d/2 + \epsilon \cdot 2^{d-2},
\end{align*}
where in the second equality, we used the fact that all reward sequences of length $d-1$ has the same marginal probability. 

Now for any history-dependent policy $\pi$, we note that
\begin{align*}
    \Exs_{\pi} \left[\tssum_{t=1}^d r_t \right] &= (d-1)/2 + \Exs_{\pi} \left[ r_d \right] \le d/2 + \epsilon \cdot 2^{d-1} \cdot \PP_{\pi} (a_{1:d} = a_{1:d}^*).
\end{align*}
Thus, for any $\epsilon$-optimal policy $\pi_{\epsilon}$ with $\epsilon < (2d)^{-2d}$, we have
\begin{align*}
    d/2 + \epsilon \cdot 2^{d-2} - \epsilon < d/2 + \epsilon \cdot 2^{d-1} \cdot \PP_{\pi} (a_{1:d} = a_{1:d}^*),
\end{align*}
which in turn implies $\PP_{\pi} (a_{1:d} = a_{1:d}^*) \ge 1/2 - 1 / 2^{d-1} \ge 1/4$.

\subsection{Proof of Lemma \ref{lemma:information_equality}}
This is a fundamental equality whose bandit version can be found in {\it e.g.,} \cite{cesa2006prediction}, \cite{garivier2019explore}. We start by unfolding the expression for KL-divergence:
\begin{align*}
    \KL &\left(\PP_{\psi}^{(1)} (\tau^{1:K}), \PP_{\psi}^{(2)} (\tau^{1:K}) \right) = \Exs_{\psi}^{(1)} \left[\log \left( \frac{\PP_{\psi}^{(1)} (\tau^{1:K})}{\PP_{\psi}^{(2)} (\tau^{1:K})} \right) \right] \\
    &= \Exs_{\psi}^{(1)} \left[\log \left( \frac{\PP_{\psi}^{(1)} (\tau^{1:K-1})}{\PP_{\psi}^{(2)} (\tau^{1:K-1})} \right) \right] + \Exs_{\psi}^{(1)} \left[\log \left( \frac{\PP_{\psi}^{(1)} (\tau^K | \tau^{1:K-1})}{\PP_{\psi}^{(2)} (\tau^K | \tau^{1:K-1})} \right) \right].
\end{align*}
Note that for any $\tau^K = (x^K_{1:H}, r^K_{1:H})$, 
\begin{align*}
    \PP_{\psi}^{(1)} (\tau^K | \tau^{1:K-1}) = \left(  \Pi_{t=1}^{H} T(s_{t}^K | x_{t-1}^K) \psi(a_t^K | h_t^K,  \tau^{1:K-1}) \right) \cdot \PP^{(1)} \left(r^K_{1:H} | x_{1:H}^K\right),
\end{align*}
and similarly, for $\PP_{\psi}^{(2)}$
\begin{align*}
    \PP_{\psi}^{(2)} (\tau^K | \tau^{1:K-1}) = \left(  \Pi_{t=1}^{H} T(s_{t}^K | x_{t-1}^K) \psi(a_t^K | h_t^K,  \tau^{1:K-1}) \right) \cdot \PP^{(2)} \left(r^K_{1:H} | x_{1:H}^K\right),
\end{align*}
which implies
\begin{align*}
    \Exs_{\psi}^{(1)} \left[\log \left( \frac{\PP_{\psi}^{(1)} (\tau^K | \tau^{1:K-1})}{\PP_{\psi}^{(2)} (\tau^K | \tau^{1:K-1})} \right) \right] &= \Exs_{\psi}^{(1)} \left[ \Exs_{\psi}^{(1)} \left[ \sum_{x_{1:H}} \log \left( \frac{\PP^{(1)} \left(r^K_{1:H} | x_{1:H}^K\right)}{\PP^{(2)} \left(r^K_{1:H} | x_{1:H}^K\right)} \right) \indic{x^K_{1:H} = x_{1:H}} \Big| \tau^{1:K-1} \right] \right] \\
    &= \sum_{x_{1:H}} \Exs_{\psi}^{(1)} \left[\log \left( \frac{\PP^{(1)} \left(r^K_{1:H} | x_{1:H}\right)}{\PP^{(2)} \left(r^K_{1:H} | x_{1:H}\right)} \right) \indic{x^K_{1:H} = x_{1:H}} \right] \\
    &= \sum_{x_{1:H}} {\rm KL} \left( \PP^{(1)} \left(\cdot | x_{1:H}\right), \PP^{(2)} \left(\cdot | x_{1:H} \right) \right) \cdot \Exs_{\psi}^{(1)} \left[\indic{x^K_{1:H} = x_{1:H}} \right],
\end{align*}
where the second equality is an application of the tower rule. Applying this recursively in $K$, we can show that
\begin{align*}
    \KL \left(\PP_{\psi}^{(1)} (\tau^{1:K}), \PP_{\psi}^{(2)} (\tau^{1:K}) \right) = \sum_{x_{1:H}} {\rm KL} \left( \PP^{(1)} \left(\cdot | x_{1:H}\right), \PP^{(2)} \left(\cdot | x_{1:H} \right) \right)  \cdot \Exs_{\psi}^{(1)} \left[\sum_{k=1}^K \indic{x^k_{1:H} = x_{1:H}} \right].
\end{align*}
By definition of $N_{\psi, x_{1:H}} (K)$, we have 
\begin{align*}
    \Exs_{\psi}^{(1)} \left[\sum_{k=1}^K \indic{x^k_{1:H} = x_{1:H}} \right] = N_{\psi, x_{1:H}} (K).
\end{align*} 
Plugging the above, we get the desired result.

\subsection{Proof of Theorem \ref{theorem:lower_bound}}
\label{appendix:proof_lower_bound}
Let $\{\mu_m^*\}_{m=1}^M$ be the specific set of vectors in $\mathbb{R}^d$ satisfying Lemma \ref{lemma:moment_matching_lower_bound} with $d = \Omega(\sqrt{M}) \ge 5$ being an odd number satisfying the condition in Lemma \ref{lemma:moment_matching_lower_bound}. Suppose the transition model follows the construction in Section \ref{section:lower_bound}: at every time step $t \in [H]$, we deterministicially move to a unique state $s_t^*$, and the reward values are binary, {\it i.e.,} $\mZ = \{0,1\}$. At every state $s_t = s_t^*$ (or time step $t$), all actions except one {\it correct} action $a_t^* \in \mA$ returns a reward sampled from a uniform distribution over $\{0,1\}$. The correct actions $a_t^*$ can be any action in $\mA$.

Consider two base systems $\mM^{(1)}$ and $\mM^{(2)}$: in $\mM^{(1)}$, reward distributions from all state-actions are uniform over $\{0,1\}$. In $\mM^{(2)}$, $\mu_m(s_t^*, a_t^*) = \mu_m^*(t)$, and otherwise uniform over $\{0,1\}$ similarly. As we can see in Lemma \ref{lemma:optimal_policy_lower_bound}, the optimal expected cumulative reward in $\mM^{(1)}$ is 1/2, whereas in $\mM^{(2)}$ optimal value is greater than $1/2 + \epsilon \cdot 2^{d-2}$. Suppose there exists a PAC-algorithm $\psi$ such that for any RMMDP instances, $\psi$ can output an $\epsilon$-optimal policy $\pi_{\epsilon}$ after $K$ episodes with probability greater than $2/3$. Then, we can use $\psi$ to test whether the system is $\mM^{(1)}$ or $\mM^{(2)}$, for any chosen optimal actions with probability greater than $2/3$. 

However, note that for any state-action sequence $x_{1:H} \neq x_{1:H}^*$, 
\begin{align*}
    {\rm KL} \left( \PP^{(1)}(\cdot| x_{1:d}), \PP^{(2)}(\cdot| x_{1:d}) \right) = 0,
\end{align*}
and
\begin{align*}
    {\rm KL} \left( \PP^{(1)}(\cdot| x_{1:d}^*), \PP^{(2)}(\cdot| x_{1:d}^*) \right) &= \sum_{r_{1:d}} \PP^{(1)}(r_{1:d}| x_{1:d}^*) \cdot \log \left( \frac{\PP^{(1)}(r_{1:d}| x_{1:d}^*)}{\PP^{(2)}(r_{1:d}| x_{1:d}^*)} \right) \\
    &= \left(\frac{1}{2}\right)^{d} \cdot \sum_{r_{1:d}} \log \left( \frac{\PP^{(1)}(r_{d}| x_{1:d}^*, r_{1:d-1})}{\PP^{(2)}(r_{d} | x_{1:d}^*, r_{1:d-1})} \right) \\
    &= \left(\frac{1}{2}\right)^{d+1} \cdot \sum_{r_{1:d}} \log \left( \frac{1/2}{1/2 + \epsilon_0} \right) + \log \left( \frac{1/2}{1/2 - \epsilon_0} \right) \\
    &= \left(\frac{1}{2}\right)^{d+1} \cdot \sum_{r_{1:d}} O(\epsilon_0^2)  = O(\epsilon_0^2),
\end{align*}
where $\epsilon_0 = \epsilon \cdot 2^{d-1}$ due to \eqref{eq:last_reward_conditional_prob}. Let $\psi'$ be an augmented exploration strategy that first runs $\psi$ for $K$ episodes and run the returned policy for $O(1/\epsilon_0^2)$ extra episodes. Let $K' = K + O(1/\epsilon_0^2)$ be the total number of episodes. We can apply Lemma \ref{lemma:information_equality} to obtain that after running an algorithm $\psi'$ for $K'$ episodes in both systems, we get
\begin{align*}
    \Exs^{(1)} \left[N_{\psi', x_{1:H}^*}(K') \right] \cdot O(\epsilon_0^2) = \KL \left(\PP_{\psi'}^{(1)} (\tau^{1:K'}), \PP_{\psi'}^{(2)} (\tau^{1:K'}) \right).
\end{align*}
By Pinsker's inequality, it holds that
\begin{align*}
    {\rm TV}\left(\PP_{\psi'}^{(1)} (\tau^{1:K'}), \PP_{\psi'}^{(2)} (\tau^{1:K'}) \right) \le \frac{1}{2} \sqrt{\KL \left(\PP_{\psi'}^{(1)} (\tau^{1:K'}), \PP_{\psi'}^{(2)} (\tau^{1:K'}) \right)}.
\end{align*}
Note that since everything is symmetric in system $\mM^{(1)}$, there exists at least one $x_{1:H}^*$ such that the expected number of the sequence being executed is small: 
\begin{align*}
    N_{\psi', x_{1:H}^*}(K') \le A^{-d} \cdot K'. 
\end{align*}
Therefore, due to LeCam's two point method \cite{lecam1973convergence}, $K'$ must satisfy that
\begin{align*}
    A^{-d} \cdot K' \cdot O(\epsilon_0^2) = \Omega(1). 
\end{align*}
This implies that $K' \ge \Omega(A^d/\epsilon_0^2) - O(1/\epsilon_0^2) = \Omega(A^d / \epsilon_0^2)$. 

Using the action amplification argument in \cite{kwon2021rl} (see their lower bound construction in Appendix), we can essentially construct the system with $O(SA/(H \log_A S))$-actions and $H = O(d \log_A S)$. If $S = \poly(A)$, this gives a lower bound $\Omega \left( \left(\frac{SA}{d} \right)^d \cdot \frac{1}{\epsilon^2} \right)$. Since $d = \Omega(\sqrt{M})$, we are done.

\end{appendices}

\end{document}